\documentclass [11pt]{article}
\usepackage{amssymb, bm, amsmath, amsfonts, mathrsfs, amsthm}
\usepackage{cite}
 \usepackage[numbers]{natbib}
\usepackage{fullpage}
\usepackage{graphicx}
\usepackage{subfig}
\usepackage{url}
\usepackage{booktabs}
\usepackage{color}
\usepackage[svgnames]{xcolor}
\usepackage{dsfont}
\usepackage{verbatim}
\usepackage{enumerate}
\usepackage{setspace}
\usepackage{multirow}
\usepackage{algpseudocode}
\usepackage{ algorithmicx}
\usepackage{algorithm}
\usepackage{qtree}
\usepackage{float}
\usepackage{morefloats}
\usepackage[mathscr]{euscript}
\usepackage{arydshln} 

\DeclareGraphicsExtensions{.pdf,.png,.jpg}

\allowdisplaybreaks

\DeclareMathOperator*{\argmax}{arg\,max}

\newcommand{\myreferences}{Draf.bbl}

\newtheorem{cor}{Corollary}

\newtheorem{lem}{Lemma}

\newtheorem{thm}{Theorem}

\newcommand\numberthis{\addtocounter{equation}{1}\tag{\theequation}}

\title{Data Clustering and Graph Partitioning via Simulated Mixing}
\author{Shahzad Bhatti\thanks{ISE \& CSL, University of Illinois at Urbana Champaign, IL Email:\{bhatti2,~ beck3,~ angelia\}@illinois.edu} \\
\and 
Carolyn Beck\footnotemark[1]
\and
Angelia Nedi\'{c}\footnotemark[1]}

\date{}
\begin{document}
\maketitle

\begin{abstract}
Spectral clustering approaches have led to well-accepted algorithms for finding accurate clusters in a given dataset. However, their application to large-scale datasets has been hindered by  computational complexity of eigenvalue decompositions. Several algorithms have been proposed in the recent past to accelerate spectral clustering, however they compromise on the accuracy of the spectral clustering to achieve faster speed. In this paper, we propose a novel spectral clustering algorithm based on a mixing process on a graph. Unlike the existing spectral clustering algorithms, our algorithm does not require computing eigenvectors. Specifically, it finds the equivalent of a linear combination of eigenvectors of the normalized similarity matrix weighted with corresponding eigenvalues. This linear combination is then used to partition the dataset into meaningful clusters. Simulations on real datasets show that partitioning datasets based on such linear combinations of eigenvectors achieves better accuracy than standard spectral clustering methods as the number of clusters increase. Our algorithm can easily be implemented in a distributed setting.
\end{abstract}

\section{Introduction}

Data clustering is a fundamental problem in pattern recognition, data mining, computer vision, machine learning, bioinformatics and several other related disciplines. It has a long history and researchers in various fields have proposed numerous solutions. Several spectral clustering algorithms have been proposed \cite{Shi2000,Ding2001,Ng2002,Bach2004}, which have enjoyed great success and have been widely used to cluster data. However, spectral clustering does not scale well to large-scale problems due to its considerable computational cost. In general, spectral clustering algorithms seek a low-dimensional embedding of the dataset by computing the eigenvectors of a Laplacian or similarity matrix. For a dataset with $n$ instances, eigenvector computation has  time complexity of $O(n^3)$ \cite{Arbenz2012}, which is significant for large-scale problems.

In the past few years, efforts have been focused towards addressing scalability of spectral clustering. A natural way to achieve scalability is to perform spectral clustering on a sample of the given dataset and then generalize the result to the rest of data. For example, Fowlkes et al. \cite{Fowlkes2004} find an approximate solution by first performing spectral clustering on a small random sample from the dataset and then using the Nystrom method; they extrapolate the solution to all the dataset. In \cite{Sakai2009}, Sakai and Imiya also find an approximate spectral clustering by clustering a random sample of the dataset. They also reduce the dimension of the dataset using random projections. Another approach proposed by Yan et al. \cite{Yan2009} works by first determining a smaller set of representative points using $k$-means (each centroid is a representative point) and then performing spectral clustering on the representative points. Finally, the original dataset is clustered by assigning each point to the cluster of its representative. In \cite{Wen-yen2011}, Chen et al. deal with large-scale data by parallelizing both computation and memory use on distributed computers. 

These methods sacrifice the accuracy of spectral clustering to achieve fast implementation. In this paper, we perform spectral clustering without explicitly calculating eigenvectors. Rather we compute a linear combination of the right-eigenvectors weighted with corresponding eigenvalues. Moreover, unlike many traditional algorithms, our algorithm does not require a predefined number of clusters, $k$, as input. This algorithm can automatically detect and adapt to any number of clusters, based on a preselected tolerance. We apply our algorithm to large size stochastic block models to illustrate the scalability and demonstrate that it can handle large datasets where traditional spectral algorithms result in memory errors. We compare the accuracy and speed of our algorithm to the normalized cut algorithm \cite{Shi2000} on real datasets and show that our approach achieves similar accuracy but at a faster speed. We also show that our algorithm is faster and more accurate than both the Nystrom method for spectral clustering \cite{Fowlkes2004} and the fast approximate spectral clustering \cite{Yan2009}.

\textbf{Notation.} Throughout this paper we use boldface to distinguish between vectors and scalars. For example, $\mathbf{v}_i$ is in boldface to identify a vector, while $x_i$, a scalar, is not boldface. We use $\mathbf{1}$ to denote the vector of ones. We denote matrices by capital letters, such as $A$, and use $a_{ij}$ to represent the entries of $A$. Calligraphic font is used to denote sets with the single exception that $\mathcal{G}$ is reserved to denote graphs. The norm $\|\cdot\|$ denotes $\|\cdot\|_2$ for vectors and for matrices it denotes spectral norm.  


\section{Problem Statement}

Consider a set of $n$ data points $\mathcal{V} = \{\mathbf{v}_1, ~\mathbf{v}_2, ~\ldots, ~\mathbf{v}_n\}$ in a $d$-dimensional space. We will often use a short-hand notation $i$ to denote the vector $\mathbf{v}_i$. The goal is to find clusters in the dataset such that points in the same cluster are similar to each other, while points in differing clusters are dissimilar under some predefined notion of similarity. In particular, suppose pairwise similarity between points is given by some similarity function $s(\mathbf{v}_i, \mathbf{v}_j)$ often abbreviated by $s(i,j)$, where usually it is assumed that the function $s$ is symmetric. Also, the similarity function $s$ is non-negative if $\mathbf{v}_i \neq \mathbf{v}_j$ and is equal to zero if $\mathbf{v}_i = \mathbf{v}_j$. A similarity matrix is an $n \times n$ symmetric matrix $W$ such that the entry $w_{ij}$ is equal to the value of the similarity function $s(\mathbf{v}_i, \mathbf{v}_j)$ between points $\mathbf{v}_i$ and $\mathbf{v}_j$.

The data points together with the similarity function form a weighted undirected similarity graph $\mathcal{G} = (\mathcal{V},\mathcal{E})$; $\mathcal{V}$ is the set of nodes (vertices) of the graph, $\mathcal{E}$ is the set of edges.  We note that the problem of finding clusters in a dataset can be converted into a graph partitioning problem. We will make this more precise in the sequel. In particular, in our case each data point $\mathbf{v}_i$ represents a vertex/node in the graph. Two vertices $i$ and $j$ are connected by an edge if their similarity $s(i,j)$ is positive and the edge weight is given by $s(i,j)$. Different similarity measures lead to different similarity graphs. The objective of constructing a similarity graph is to model the local neighborhood relationships  to capture the geometric structure of the dataset using the similarity function. Some commonly used similarity graphs for spectral clustering are noted below.
\begin{enumerate}[a)]
\item \textbf{Gaussian similarity graphs:} Gaussian similarity graphs are based on the distance between points. Typically every pair of vertices is connected by an edge and the edge weight is determined by the Gaussian function (radial basis function) $s(i,j) = exp(- \|\mathbf{v}_i - \mathbf{v}_j\|^2 / 2\sigma^2 )$, where parameter $\sigma$ controls how quickly the similarity fades away as the distance between the points increases. It is often helpful to remove the edges for edge weights below a certain threshold (say $\delta$) to sparsify the graph. Thus the similarity function for this type of graph is 
\[
s(i,j) = \left \{ \begin{array} {l l}
e^{- \frac{\|\mathbf{v}_i - \mathbf{v}_j\|^2}{ 2\sigma^2}}  & \text{if } \|\mathbf{v}_i - \mathbf{v}_j\| > \delta, \\
0 & \text{otherwise}.
\end{array}\right.
\]
\item \textbf{$p$-nearest neighbor graphs:} As the name suggests each vertex $\mathbf{v}_i$ is connected to its $p$ nearest neighbors, where the nearness between $i$ and $j$ is measured by the distance $\|\mathbf{v}_i - \mathbf{v}_j\|$. This similarity measure results in a graph which is not necessarily symmetric in its similarity function, since the nearness relationship is not symmetric. In particular, if $\mathbf{v}_j$ is among the $p$ nearest neighbors of $\mathbf{v}_i$ then it is not necessary for $\mathbf{v}_i$ to be among the $p$ nearest neighbors of $\mathbf{v}_j$. We make the similarity measure symmetric by placing an edge between two vertices $\mathbf{v}_i$ and $\mathbf{v}_j$ if either $\mathbf{v}_i$ is among the $p$ nearest neighbors of $\mathbf{v}_j$ or $\mathbf{v}_j$ is among the $p$ nearest neighbors of $\mathbf{v}_i$, that is 
\[
s(i,j) = \left \{ \begin{array} {l l}
1  & \text{if either $i$ or $j$ is one of the $p$ nearest neighbors of the other}, \\
0 & \text{otherwise}.
\end{array}\right.
\]
\item \textbf{$\epsilon$-neighborhood graphs:} In an $\epsilon$-neighborhood graph, we connect two vertices $\mathbf{v}_i$ and $\mathbf{v}_j$ if the distance $\| \mathbf{v}_i - \mathbf{v}_j\|$ is less than $\epsilon$ giving us the following similarity function.
\[
s(i,j) = \left \{ \begin{array} {l l}
1  & \text{if } \|\mathbf{v}_i - \mathbf{v}_j \| \le \epsilon, \\
0 & \text{otherwise}.
\end{array}\right.
\]
\end{enumerate}
Thus finding clusters in a dataset is equivalent to finding partitions in the similarity graph such that the sum of edge weights between partitions is small and the partitions themselves are dense subgraphs. The degree of each vertex $i$ of the graph is given by $d_{i} = \sum_{j=1}^n w_{ij}$. A degree matrix $D$ is a diagonal matrix with its diagonal elements given by $d_i$ for $ i = 1, \ldots, n$. We assume that the degree of each vertex is positive. This in turn allows us to define the normalized similarity matrix by $\overline{W} = D^{-1}W$. \par

 \section{Mixing processes}
 
As a visualization of our clustering approach, consider a mixing process in which one imagines that every vertex in the graph moves towards (mixes with) other vertices in discrete time steps. At each time step, vertex $i$ moves towards (mixes with) vertex $j$ by a distance  proportional to the similarity $s(i,j)$. Thus the larger the similarity $s(i, j)$, the larger the distance vertices $i$ and $j$ move towards each other i.e., the greater the mixing. Moreover, a point $\mathbf{v}_i$ will move away from the points which have weak similarity with it. Thus, similar points will move towards each other making dense clusters and dissimilar points will move away from each other increasing the separability between clusters. Clusters in this transformed distribution of points then can easily be identified by the $k$-means algorithm. \par

To describe the above idea more precisely, consider the following model, where each point $\mathbf{v}_i$ moves according to the following equation, starting at its original position at time $t = 0$:
\begin{align}
\mathbf{v}_i^{t+1} &= \mathbf{v}_i^{t} + \alpha \sum_{j=1}^n \overline{w}_{ij} (\mathbf{v}_j^t - \mathbf{v}_i^t) \nonumber \\
&=  (1 - \alpha)\mathbf{v}_i^t + \alpha \sum_{j=1}^n \overline{w}_{ij} \, \mathbf{v}_j^t.   \label{Eq: TranSim}
\end{align}

The parameter $\alpha \in [0, 1]$ is the step size, which controls the speed of movement (or mixing rate) in each time interval. Observe that if the underlying graph has a bipartite component and $\alpha = 1$, then in each time step all points on one side of this component would move to the other side and vice versa. Therefore, points in this component would not actually mix even after a large number of iterations (for details see \cite{Chung1997}). For such graphs we must have $\alpha$ bounded away from 1. We can use $\alpha = 1$ for graphs without a bipartite component. Assuming each point $\mathbf{v}_i$ is a row vector, we express equation \eqref{Eq: TranSim} in a matrix form:
\begin{align}
V^{t+1} &= ((1 - \alpha) I + \alpha \overline{W}) V^t  \nonumber\\
&= M V^t.    \label{Eq: TranMat}
\end{align}
The matrix $V^{t}$ is a $n \times d$ matrix with row $i$ representing the position of point $\mathbf{v}_i$ at time $t$. $I$ is a $n \times n$ identity matrix, and we define $M = (1 - \alpha) I + \alpha \overline{W}$. Note that the matrix $M$ is essentially the transition matrix of a lazy random walk with probability of staying in place given by $1 - \alpha$. Since $M$ also captures  the similarity of the data points, one would expect that for $t$ large enough, the process in equation \eqref{Eq: TranMat} would reveal the data clusters, since $M$ will mix the data points according to their similarities. Using this intuition, one can expect that a heuristic algorithm based on equation \eqref{Eq: TranMat} can be constructed to determine the clusters, as given in the following algorithm.
 
\begin{algorithm}
\algrenewcommand\algorithmicrepeat{\textbf{Loop over}}
\caption{Point-Based Resource Diffusion (PRD) }\label{Alg: points}
\begin{algorithmic}[1]
\State{\textbf{Input:} Set of data points $\mathcal{V}$, number of clusters $k$}
\State{Represent the data points in the matrix $V$ with point $\mathbf{v}_i$ being the $i$th row.}
\State {Compute $M = (1- \alpha) I + \overline{W}$.}
\Repeat $~t$
\State $V^{t+1} \gets M V^t $
\Until{ Stopping criteria is met.}
\State{Find $k$ clusters from rows of $V^{t+1}$ using $k$-means algorithm.}
\State{\textbf{Output:} Clustering obtained in the final iteration.}
\end{algorithmic}
\end{algorithm}
Algorithm \ref{Alg: points} has two limitations: a) it does not scale well with the dimension $d$ of the data points, because the number of computations in each iteration is  $O(n^2d)$,  and b) it fails to identify clusters contained within other clusters. For example, in the case of two concentric circular clusters, points in both clusters will move towards the center and become one cluster, losing the geometric structure inherent in the data. Hence it becomes impossible to discern these clusters using the $k$-means algorithm in Step 7. To overcome these limitations, we associate an \emph{agent} $x_i$ to each point $\mathbf{v}_i$ and carry out calculations in the \emph{agent space}. Agents are generated by choosing $n$ points uniformly at random from a bounded interval $[0,b]$\footnote{Note, $b$ is a scaling parameter and does not change the resulting clustering. For the sake of simplicity we use a probability vector for the analysis in the subsequent section. }. We rewrite equation \eqref{Eq: TranMat} in the agent space as:
\begin{align}
\mathbf{x}^{t+1} &= ((1 - \alpha) I + \alpha \overline{W}) \mathbf{x}^t  = M \mathbf{x}^t.    \label{Eq: Lazy}
\end{align}
We refer to this iterative equation as the Mixing Process. In the following section we analyze this process using the properties of the random walk matrix $M$. 

\section{Analysis of the Mixing Process}
The matrix $M$ captures the similarity structure of the data, and the idea behind using the iterative process \eqref{Eq: Lazy} is that, after some sufficient number of iterations, the entries of the vector $\mathbf{x}^{t+1}$ will reveal clusters on a real line, which will be representative of the clusters in the data. The fact is that the process \eqref{Eq: TranMat} and \eqref{Eq: Lazy} both mix with the same speed, which is governed by the random walk $M$. Thus, the hope is that through the process in \eqref{Eq: Lazy} we determine the weakly coupled components in the matrix $M$, which can lead us to the data clusters of the points $\mathbf{v}_1, \ldots , \mathbf{v}_n$.

The Mixing Process shares a resemblance to the power iteration. Unlike the power iteration, we here use the mixing process to discover strongly coupled components of $M$, which translate to data clusters.

\subsection{Properties of the matrix $M$}
We first show that the matrix $M$ is diagonalizable, which allows for a more straightforward analysis. By definition,
\begin{align}
M &= (1 - \alpha) I + \alpha \overline{W}  \nonumber \\
& = I - \alpha(I - D^{-1}W) \nonumber \\
&= I - \alpha D^{-1/2} L D^{1/2} \nonumber\\
& = D^{-1/2} (I - \alpha L) D^{1/2}\label{Eq: nl},
\end{align}
where $L = I - D^{-1/2}W D^{-1/2}$ is the normalized Laplacian of the graph $\mathcal{G}$. Let $\bm{\phi}_i$ be a right eigenvector of $L$ with eigenvalue $\lambda_i$, then $D^{-1/2}\bm{\phi}_i$ is a right eigenvector of $M$ with eigenvalue $\mu_i = 1 - \alpha \lambda_i$, that is
\begin{align*}
MD^{-1/2}\bm{\phi}_i & = (I - \alpha D^{-1/2}L D^{1/2})D^{-1/2}\bm{\phi}_i \nonumber \\
& = (1 -\alpha \lambda_i)  D^{-1/2}\bm{\phi}_i .
\end{align*}
This gives us a useful relationship between the spectra of the random walk matrix $M$ and the normalized Laplacian $L$. It is well known that the eigenvalues of a normalized Laplacian lie in the interval $[0, 2]$, see for example \cite{Chung1997}. Thus, if $0 = \lambda_1 \le \lambda_2 \le \ldots \le \lambda_n \le 2$ are the eigenvalues of $L$, then the corresponding eigenvalues of $M$ are $1 = \mu_1 \ge \mu_2 \ge \ldots \ge \mu_n \ge 1 - 2\alpha.$ It is worth noting that we are considering the \emph{right} eigenvector of the random walk matrix $M$. One should not confuse this with the left eigenvector. \par

Although the matrix $M$ is not symmetric, $L$ is a symmetric positive semi-definite matrix, thus its normalized eigenvectors form an orthonormal basis for $\mathbb{R}^n$ and we can express $L$ in the following form:
\begin{equation}
L = \sum_{i=i}^n \lambda_i \bm{\phi}_i \bm{\phi}_i^T. \label{Eq: LapDec}
\end{equation}
Using \eqref{Eq: nl} and \eqref{Eq: LapDec}, we obtain
\begin{align}
M^t &= \left(D^{-1/2} (I - \alpha L) D^{1/2} \right)^t \nonumber\\
& =  D^{-1/2} \left(  \sum_{i=i}^n (1 - \alpha \lambda_i)^t \bm{\phi}_i \bm{\phi}_i^T \right) D^{1/2}. \label{Eq: Mt}
\end{align}
We will exploit this relationship in our proofs later.


\subsection{The ideal case}
For the sake of analysis, it is worthwhile to consider the ideal case, in which all points form tight clusters that are well-separated. By well-separated, we mean that if points $\mathbf{v}_i$ and $\mathbf{v}_j$ lie in different clusters, then their similarity $w_{ij} = 0$. Suppose that the data consists of $k$ clusters $\mathcal{V}_1,\, \mathcal{V}_2,\, \ldots, \, \mathcal{V}_k$ with $n_1,\, n_2, \, \ldots,\, n_k$ points, respectively, such that $\cup_{i=1}^{k} \mathcal{V}_i = \mathcal{V}$ and $n = \sum_{i=1}^k n_i$. For the ease of exposition, we also assume that the $\mathbf{v}_i$'s are numbered in such a way that points $\mathbf{v}_1,\mathbf{v}_2, \ldots, \mathbf{v}_{n_1}$ are in cluster $\mathcal{V}_1$, the points $\mathbf{v}_{n_1+1},\mathbf{v}_{n_1+2}, \ldots, \mathbf{v}_{n_1+ n_2}$ are in cluster $\mathcal{V}_2$ and so on. \par
 
The underlying graph in the ideal case consists of $k$ connected components $\mathcal{G}_1, \, \mathcal{G}_2, \, \ldots, \, \mathcal{G}_k$, where each component $\mathcal{G}_j$ consists of vertices in the corresponding cluster $\mathcal{V}_j$. We represent this ideal graph by $\mathcal{G}^*$, its normalized Laplacian by $L^*$ and its similarity matrix by $W^*$. The $n$-dimensional characteristic vector $\bm{\chi}_j$ of the $j$th component $\mathcal{G}_j$ is defined as 
\[
\bm{\chi}_j(i) = \left\{ 
\begin{array}{l @{\qquad} l}
1 & \text{if } \mathbf{v}_i \in \mathcal{G}_j, \\
0 & \text{otherwise}.
\end{array}
\right.
\]
The ideal similarity matrix $W^*$ and, consequently the ideal normalized Laplacian $L^*$ of a graph with $k$ connected components, are both block-diagonal with the $j^{th}$ block representing the component $\mathcal{G}_j$, i.e., 
\begin{align*}
W^* &= diag(W_1, \, W_2, \, \ldots, \, W_k),  \\
L^* &= diag(L_1, \, L_2, \, \ldots, \, L_k).
\end{align*}
Since $L^*$ is block-diagonal, its spectrum is the union of spectra of $L_1, \, L_2, \, \ldots,  \, L_k$. The eigenvalue $\lambda_1 = 0$ of $L^*$ has multiplicity $k$ with $k$ linearly independent normalized eigenvectors $\bm{\phi}_1^*, \, \bm{\phi}_2^*, \, \ldots, \, \bm{\phi}_k^*$.  Each of these eigenvectors is given by $\bm{\phi}_j^* = D^{1/2}\bm{\chi}_j/\| D^{1/2}\bm{\chi}_j\|$. In the following theorem, we prove that if we have an ideal graph then the iterate sequence $\{\mathbf{x}^t\}$ generated by the Mixing Process \eqref{Eq: Lazy} converges to a linear combination of characteristic vectors $\bm{\chi}_j$'s of the $k$ components of the graph. The $\bm{\chi}_j$'s are also eigenvectors of $M^*$, where $M^* = (1 - \alpha)I - D^{-1}W^*$, corresponding to first $k$ eigenvalues $\mu_1 = \mu_2 = \,... = \mu_k = 1$. 

\begin{thm}
Suppose that we have an ideal dataset which consists of $k$ clusters as defined previously and let $\mathbf{x}^0$ be any vector such that each $x_i^0>0$ and $(\mathbf{x}^0)^T \mathbf{1} = 1$, then 

$$ \| {M^*}^t \mathbf{x}^0 - \sum_{i = 1}^k c_i \bm{\chi}_i\| \le \max_{i > k} | 1 - \alpha \lambda_i |^t \frac{\max_j \sqrt{d_j}}{\min_j \sqrt{d_j}}, $$
 where $c_i =  \frac{\bm{\chi}_i^T D \mathbf{x}^0 }{ \mathbf{1}^TD \bm{\chi}_i}$and $d_j$ is the degree of the $j^{th}$ node.
\end{thm}
\begin{proof}
We begin by noting that the iterative process \eqref{Eq: Lazy} is equivalent to $\mathbf{x}^{t} = M^t \mathbf{x}^0$. Using \eqref{Eq: Mt}, we have
\begin{align*}
\| {M^*}^t \mathbf{x}^0 - \sum_{i = 1}^k c_i \bm{\chi}_i\| 
& = \|D^{-1/2}\left(\sum_{i=1}^n (1 - \alpha \lambda_i)^t \bm{\phi}_i^* {\bm{\phi}_i^*}^T\right) D^{1/2}\mathbf{x}^0 -  \sum_{i = 1}^k c_i \bm{\chi}_i \|.    
\end{align*}
Separating the first $k$ terms in the sum and using the fact that eigenvector $\bm{\phi}_i^* = \frac{D^{1/2} \bm{\chi}_i}{\| D^{1/2} \bm{\chi}_i\|},$ for $i = 1, \ldots, k$, the above equation can be simplified as
\begin{align*}
\| {M^*}^t \mathbf{x}^0 - \sum_{i = 1}^k c_i \bm{\chi}_i\| 
& = \|D^{-1/2}\left(\sum_{i = 1}^k (1 - \alpha \lambda_i)^t \bm{\phi}_i^* {\bm{\phi}_i^*}^T\right) D^{1/2}\mathbf{x}^0 \\
&~~~+ D^{-1/2}\left(\sum_{i = k+1}^n (1 - \alpha \lambda_i)^t \bm{\phi}_i^* {\bm{\phi}_i^*}^T\right) D^{1/2}\mathbf{x}^0 -  \sum_{i = 1}^k c_i \bm{\chi}_i \|    \\
& = \|\sum_{i = 1}^k D^{-1/2} \bm{\phi}_i^* {\bm{\phi}_i^*}^T D^{1/2}\mathbf{x}^0 + D^{-1/2}\left(\sum_{i = k+1}^n (1 - \alpha \lambda_i)^t \bm{\phi}_i^* {\bm{\phi}_i^*}^T\right) D^{1/2}\mathbf{x}^0 -  \sum_{i = 1}^k c_i \bm{\chi}_i \|    \\
& = \|\sum_{i = 1}^k D^{-1/2} \frac{D^{1/2} \bm{\chi}_i}{\| D^{1/2} \bm{\chi}_i\|}\frac{(D^{1/2} \bm{\chi}_i)^T}{\| D^{1/2} \bm{\chi}_i\|} D^{1/2}\mathbf{x}^0 \\
&~~~+ D^{-1/2}\left(\sum_{i = k+1}^n (1 - \alpha \lambda_i)^t \bm{\phi}_i^* {\bm{\phi}_i^*}^T\right) D^{1/2}\mathbf{x}^0 -  \sum_{i = 1}^k c_i \bm{\chi}_i \|.  
\end{align*}
We can simplify the first term in the norm as
\begin{align*}
D^{-1/2} \frac{D^{1/2} \bm{\chi}_i}{\| D^{1/2} \bm{\chi}_i\|}\frac{(D^{1/2} \bm{\chi}_i)^T}{\| D^{1/2} \bm{\chi}_i\|} D^{1/2}\mathbf{x}^0 & = \frac{ \bm{\chi}_i ~ \bm{\chi}_i^T D^{1/2} D^{1/2} \mathbf{x}^0}{\| D^{1/2} \bm{\chi}_i\|^2} \\
&= \frac{ \bm{\chi}_i^T D \mathbf{x}^0}{\| D^{1/2} \bm{\chi}_i\|^2} \bm{\chi}_i,
\end{align*}
where the term $\|  D^{1/2} \bm{\chi}_i\|^2$ is equal to the sum of the degrees of vertices in component $\mathcal{G}_i$ (commonly called volume of $\mathcal{G}_i$) which can also be expressed by $ \mathbf{1}^TD \bm{\chi}_i$, to have
\begin{align*}
\| {M^*}^t \mathbf{x}^0 - \sum_{i = 1}^k c_i \bm{\chi}_i\|  & = \|\sum_{i = 1}^k \frac{\bm{\chi}_i^T D \mathbf{x}^0 }{ \mathbf{1}^TD \bm{\chi}_i} \bm{\chi}_i + D^{-1/2}\left(\sum_{i = k+1}^n (1 - \alpha \lambda_i)^t \bm{\phi}_i^* {\bm{\phi}_i^*}^T\right) D^{1/2}\mathbf{x}^0 -  \sum_{i = 1}^k c_i \bm{\chi}_i \|   \\
& = \| D^{-1/2}\left(\sum_{i = k+1}^n (1 - \alpha \lambda_i)^t \bm{\phi}_i^* {\bm{\phi}_i^*}^T\right) D^{1/2}\mathbf{x}^0 \|. 
\end{align*}
In the last equation we have used $c_i =  \frac{\bm{\chi}_i^T D \mathbf{x}^0 }{ \mathbf{1}^TD \bm{\chi}_i}$. Now using the properties of the norm, we can separate the terms, giving us
 \begin{align*}
\| {M^*}^t \mathbf{x}^0 - \sum_{i = 1}^k c_i \bm{\chi}_i\| & \le \| D^{-1/2}\| ~\|\sum_{i = k+1}^n (1 - \alpha \lambda_i)^t \bm{\phi}_i^* {\bm{\phi}_i^*}^T\|~\| D^{1/2}\|~\|\mathbf{x}^0 \| \\
&  \le \max_{i > k} | 1 - \alpha \lambda_i |^t \frac{\max_j \sqrt{d_j}}{\min_j \sqrt{d_j}}, ~~~~ \text{ since } \|\mathbf{x}^0 \| \le 1. 
\end{align*}
\end{proof}

Note that we can always choose $\alpha \in [0, 1]$ such that $\lambda_{k+1} = \argmax_{i > k} | 1 - \alpha \lambda_i |$. Thus the preceding inequality can be written as
\begin{align*}
\| {M^*}^t \mathbf{x}^0 - \sum_{i = 1}^k c_i \bm{\chi}_i\|& \le ( 1 - \alpha \lambda_{k+1})^t \frac{\max_j \sqrt{d_j}}{\min_j \sqrt{d_j}} \\
& \le e^{-\alpha t {\lambda}_{k+1}} \frac{\max_j \sqrt{d_j}}{\min_j \sqrt{d_j}}.
\end{align*}
For any $\xi >0,$ there exists some $t>0$ such that
\begin{align}
e^{-\alpha t {\lambda}_{k+1}} \frac{\max_j \sqrt{d_j}}{\min_j \sqrt{d_j}} &\le \xi. \label{Eq: epsGe}
\end{align}
Specifically, taking the log and simplifying, we have
\begin{align*}
 \frac{1}{\alpha {\lambda}_{k+1}}  \ln \left(\frac{\max_j \sqrt{d_j}}{\xi ~\min_j \sqrt{d_j}}\right) &\le t.
\end{align*}

\subsection{The general case}

In practice, the graph under consideration may not have $k$ connected components, but rather $k$ nearly connected components i.e., $k$ dense subgraphs sparsely connected by bridges (edges). We can obtain $k$ connected components from such a graph by removing only a small fraction of edges. This means that matrices $W$ and  $L$ have non-zero off-diagonal blocks, but both matrices have dominant blocks on the diagonal. The general case is thus a perturbed version of the ideal case. 

Let $W = W^* + E$ be the similarity matrix for a dataset $\mathcal{V}$, where $W^*$ is the similarity matrix corresponding to the true clusters (an ideal similarity  matrix), which is block-diagonal and symmetric. We obtain $W^*$ by replacing the off-diagonal block elements of $W$ with zeros and adding the sum of the off-diagonal block weights in each row to the diagonal elements. This results in the matrices $W$ and $W^*$ having the same degree matrix $D$. The matrix $E$ is then a symmetric matrix with row and column sums equal to zero with the $i^{th}$ diagonal entry given by $e_{ii}=-\sum_{j=1}^n e_{ij}$. The off-diagonal entries of $E$ are the same as the entries in the off-diagonal blocks of $W$. For example, for

\[
W = \left[ \begin{array}{c@{\quad} c @{\quad}c : c @{\quad}c @{\quad}c}
0 & 20 & 50 & 1 & 2 & 1 \\
20 & 0 & 30 & 0 & 1 & 1 \\
50 & 30  & 0 & 1 & 0 & 1 \\ \hdashline
1 & 0 & 1 & 0 & 25 & 40 \\
2 & 1 & 0 & 25 & 0 & 30 \\
1 & 1 & 1 & 40 & 30 & 0
\end{array}\right],
\]

the matrices $W^*$ and $E$  satisfying the above constraints are given by

\[
W^*= \left[ \begin{array}{c@{\quad} c @{\quad}c : c @{\quad}c @{\quad}c}
4 & 20 & 50 & 0 & 0 & 0 \\
20 & 2 & 30 & 0 & 0 & 0 \\
50 & 30  & 2 & 0 & 0 & 0 \\ \hdashline
0 & 0 & 0 & 2 & 25 & 40 \\
0 & 0 & 0 & 25 & 3 & 30 \\
0 & 0 & 0 & 40 & 30 & 3
\end{array}\right],
\text { and }
E = \left[ \begin{array}{r@{\quad} r @{\quad}r : r @{\quad}r @{\quad}r}
-4 & 0 & 0 & 1 & 2 & 1 \\
0 & -2 & 0 & 0 & 1 & 1 \\
 0 & 0 & -2 & 1 & 0 & 1 \\  \hdashline
1 & 0 & 1 & -2 & 0 & 0  \\
2 & 1 & 0 & 0 & -3 & 0 \\
1 & 1 & 1 & 0 & 0 & -3
\end{array}\right].
\]

Using the definition of $L$, one can then show that the following result holds.
\begin{lem}
If $W = W^* + E$ is a similarity matrix for a dataset $\mathcal{V}$, then the normalized Laplacian of the corresponding graph is $L = L^* - D^{-1/2}ED^{-1/2}$, where $L^*$ is the normalized Laplacian of the ideal graph corresponding to the true clusters.
\end{lem}	
Since eigenvalues and eigenvectors are continuous functions of entries of a matrix, the eigenvalues $\lambda_i$'s of $L$ can be written as
\[
\lambda_i  = \lambda_i^* + \tilde{\lambda}_i,
\]
where $\lambda_i^*$ is the eigenvalue of the ideal normalized Laplacian $L^*$, and $\tilde{\lambda}_i$ depend continuously on the entries of $\bar{E} \triangleq D^{-1/2}ED^{-1/2}$. Similarly the eigenvectors $\bm{\phi}_i$'s of $L$ can be expressed as
\[
\bm{\phi}_i = \bm{\phi}_i^* + \bm{\tilde{\phi}}_i,
\]
where $\bm{\phi}_i^*$ is the eigenvector of the ideal normalized Laplacian $L^*$ and $\bm{\tilde{\phi}}_i$ depend continuously on the entries of $\bar{E}$. Note that the pair $(\tilde{\lambda}_i, \bm{\tilde{\phi}}_i)$ is not necessarily an eigenvalue/eigenvector pair of $\bar{E}$. We assume that $\|E\|$ and consequently $\|\bar{E}\|$ are small enough so that $|\tilde{\lambda}_i|$ and $\| \bm{\tilde{\phi}}_i\|$ are also small.

\begin{thm} \label{Thm: gen}
Suppose that we have a dataset which consists of $k$ clusters and let $\mathbf{x}^0$ be any vector such that each $x_i^0>0$ and $(\mathbf{x}^0)^T \mathbf{1} = 1$, then we have 

\[
\| M^t \mathbf{x}^0 - \sum_{i=1}^k c_i \bm{\chi}_i \|   \le    \left( \sum_{i=1}^k  \left( 2\| \bm{\tilde{\phi}}_i \| +  \|\bm{\tilde{\phi}}_i \|^2  \right)  + \max_{\ell > k} |1 - \alpha \lambda_{\ell}|^t \right) \frac{\max_j\sqrt{d_j}}{\min_j \sqrt{d_j}},
\]
where $c_i =  (1 - \alpha \tilde{\lambda}_i)^{t} \frac{\bm{\chi}_i^T D \mathbf{x}^0 }{ \mathbf{1}^TD \bm{\chi}_i}$ and $d_j$ is the degree of the $j^{th}$ node.
\end{thm}
\begin{proof}
Using \eqref{Eq: Mt} and separating the first $k$ terms, we get
\begin{align*}
\| M^t \mathbf{x}^0 - \sum_{i=1}^k c_i \bm{\chi}_i \| & = \|D^{-1/2}\sum_{i=1}^n(1 - \alpha \lambda_i)^t  \bm{\phi}_i \bm{\phi}_i^T D^{1/2} \mathbf{x}^0 - \sum_{i=1}^k c_i \bm{\chi}_i \|\\
& = \|D^{-1/2} \sum_{i=1}^k (1 - \alpha \lambda_i)^t  \bm{\phi}_i \bm{\phi}_i^T D^{1/2} \mathbf{x}^0 \\
&~~~+ D^{-1/2}\sum_{i=k+1}^n(1 - \alpha \lambda_i)^t  \bm{\phi}_i \bm{\phi}_i^T D^{1/2} \mathbf{x}^0 - \sum_{i=1}^k c_i \bm{\chi}_i \|.
\end{align*}
Substituting $\bm{\phi}_i = \bm{\phi}_i^* + \bm{\tilde{\phi}}_i$ and $\lambda_i = \lambda_i^* + \tilde{\lambda}_i$ in the above equation and using the fact that $\lambda_i^* = 0$ for $i=1, \ldots, k,$ we have
\begin{align}
\| M^t \mathbf{x}^0 - \sum_{i=1}^k c_i \bm{\chi}_i \| & = \|D^{-1/2}\sum_{i=1}^k(1 - \alpha \tilde{\lambda}_i)^t \left( \bm{\phi}_i^* {\bm{\phi}_i^*}^T + \bm{\phi}_i^* \bm{\tilde{\phi}}_i^T +  \bm{\tilde{\phi}}_i {\bm{\phi}_i^*}^T + \bm{\tilde{\phi}}_i {\bm{\tilde{\phi}}_i}^T \right) D^{1/2} \mathbf{x}^0 \nonumber \\
&~~~+ D^{-1/2}\sum_{i=k+1}^n(1 - \alpha \lambda_i)^t  \bm{\phi}_i \bm{\phi}_i^T D^{1/2} \mathbf{x}^0  - \sum_{i=1}^k c_i \bm{\chi}_i \| \nonumber \\
&=\| \sum_{i=1}^k (1 - \alpha \tilde{\lambda}_i)^t  D^{-1/2} \frac{D^{1/2}\bm{\chi}_i(D^{1/2}\bm{\chi}_i)^T}{\|D^{1/2}\bm{\chi}_i\|^2} D^{1/2} \mathbf{x}^0  \nonumber \\
& ~~~+ \sum_{i=1}^k (1 - \alpha \tilde{\lambda}_i)^t  D^{-1/2}\left(\bm{\phi}_i^* \bm{\tilde{\phi}}_i^T +  \bm{\tilde{\phi}}_i {\bm{\phi}_i^*}^T + \bm{\tilde{\phi}}_i {\bm{\tilde{\phi}}_i}^T\right) D^{1/2} \mathbf{x}^0  \nonumber  \\
&~~~+ D^{-1/2}\sum_{i=k+1}^n (1 - \alpha \lambda_i)^t  \bm{\phi}_i \bm{\phi}_i^T D^{1/2} \mathbf{x}^0 - \sum_{i=1}^k c_i \bm{\chi}_i \|, \label{Eq: ci}
\end{align}

where 
\begin{align*}
(1 - \alpha \tilde{\lambda}_i)^t D^{-1/2} \frac{D^{1/2}\bm{\chi}_i(D^{1/2}\bm{\chi}_i)^T}{\|D^{1/2}\bm{\chi}_i\|^2} D^{1/2} \mathbf{x}^0 & = (1 - \alpha \tilde{\lambda}_i)^t  \frac{\bm{\chi}_i ~ \bm{\chi}_i^T D^{1/2} D^{1/2} \mathbf{x}^0}{\|D^{1/2}\bm{\chi}_i\|^2} \\
&= (1 - \alpha \tilde{\lambda}_i)^t  \frac{\bm{\chi}_i^T D \mathbf{x}^0}{\mathbf{1}^T D \bm{\chi}_i} \bm{\chi}_i  \\
&= c_i \bm{\chi}_i.
\end{align*}

Substituting the above expression in \eqref{Eq: ci} and using the triangle inequality, we have
\begin{align*}
\| M^t \mathbf{x}^0 - \sum_{i=1}^k c_i \bm{\chi}_i \| &= \| \sum_{i=1}^k (1 - \alpha \tilde{\lambda}_i)^t  D^{-1/2}\left(\bm{\phi}_i^* \bm{\tilde{\phi}}_i^T +  \bm{\tilde{\phi}}_i {\bm{\phi}_i^*}^T + \bm{\tilde{\phi}}_i {\bm{\tilde{\phi}}_i}^T\right) D^{1/2} \mathbf{x}^0  \\
&~~~+ D^{-1/2}\sum_{i=k+1}^n (1 - \alpha \lambda_i)^t  \bm{\phi}_i \bm{\phi}_i^T D^{1/2} \mathbf{x}^0 \|  \\
& \le \| \sum_{i=1}^k (1 - \alpha \tilde{\lambda}_i)^t  D^{-1/2}\left(\bm{\phi}_i^* \bm{\tilde{\phi}}_i^T +  \bm{\tilde{\phi}}_i {\bm{\phi}_i^*}^T + \bm{\tilde{\phi}}_i {\bm{\tilde{\phi}}_i}^T\right) D^{1/2} \mathbf{x}^0  \| \\
&~~~+ \| D^{-1/2}\sum_{i=k+1}^n (1 - \alpha \lambda_i)^t  \bm{\phi}_i \bm{\phi}_i^T D^{1/2} \mathbf{x}^0 \|  \\
& \le  \sum_{i=1}^k | 1 - \alpha \tilde{\lambda}_i |^t  \| D^{-1/2} \| \left( \| \bm{\phi}_i^* \bm{\tilde{\phi}}_i^T \| +  \|\bm{\tilde{\phi}}_i {\bm{\phi}_i^*}^T \| + \| \bm{\tilde{\phi}}_i {\bm{\tilde{\phi}}_i}^T \| \right)  \| D^{1/2} \| ~ \|\mathbf{x}^0  \| \\
&~~~+ \| D^{-1/2} \|~ \| \sum_{i=k+1}^n (1 - \alpha \lambda_i)^t  \bm{\phi}_i \bm{\phi}_i^T \|  
~\| D^{1/2} \| ~\| \mathbf{x}^0 \|.  \numberthis \label{Eq: subD} 
\end{align*}
Since $\| \mathbf{x}^0\| \le 1, ~\|\bm{\phi}^*_i\| = 1$, ~$|1 - \alpha \tilde{\lambda}_i| \le 1$,  we can further simplify inequality \eqref{Eq: subD} as
\begin{align*}
\| M^t \mathbf{x}^0 - \sum_{i=1}^k c_i \bm{\chi}_i \| & \le  \sum_{i=1}^k | 1 - \alpha \tilde{\lambda}_i |^t  \| D^{-1/2} \| \left( 2\| \bm{\tilde{\phi}}_i \| +  \|\bm{\tilde{\phi}}_i \|^2  \right)  \| D^{1/2} \| \\
&~~~+ \| D^{-1/2} \|~ \| \sum_{i=k+1}^n  (1 - \alpha \lambda_i)^t  \bm{\phi}_i \bm{\phi}_i^T \|  
~\| D^{1/2} \| \\
& \le  \sum_{i=1}^k   \left( 2\| \bm{\tilde{\phi}}_i \| +  \|\bm{\tilde{\phi}}_i \|^2  \right) \frac{\max_j \sqrt{d_j}}{\min_j \sqrt{d_j}} + \max_{\ell > k} |1 - \alpha \lambda_{\ell}|^t \frac{\max_j\sqrt{d_j}}{\min_j\sqrt{d_j}} \\
&= \left( \sum_{i=1}^k  \left( 2\| \bm{\tilde{\phi}}_i \| +  \|\bm{\tilde{\phi}}_i \|^2  \right)  + \max_{\ell > k} |1 - \alpha \lambda_{\ell}|^t \right) \frac{\max_j\sqrt{d_j}}{\min_j \sqrt{d_j}}. \numberthis  \label{Eq: th2} 
\end{align*}
\end{proof}
Note that we can always choose $\alpha$ such that $ (1 - \alpha \lambda_{k+1}) =  \max_{\ell > k} |1 - \alpha \lambda_{\ell}|$ and the above expression then becomes
\[
\| M^t \mathbf{x}^0 - \sum_{i=1}^k c_i \bm{\chi}_i \|   \le    \left( \sum_{i=1}^k  \left( 2\| \bm{\tilde{\phi}}_i \| +  \|\bm{\tilde{\phi}}_i \|^2  \right)  + (1 - \alpha \lambda_{k+1})^t \right) \frac{\max_j\sqrt{d_j}}{\min_j \sqrt{d_j}}.
\]

Observe that $\lambda_i = \tilde{\lambda}_i$ for $i = 1,\ldots,k$. Thus, assuming that the perturbation is small, the first $k$ eigenvalues of the Laplacian are close to zero. If the eigengap $(\lambda_{k+1} - \lambda_k)$ is large enough, then for some $t>0$, we will have both $(1- \alpha \lambda_k) = (1 -  \alpha \tilde{\lambda}_k)^t \ge 1 - \delta$ for a small $\delta >0$, and $(1 - \alpha \lambda_{k+1})^t \le \epsilon$ for a small $\epsilon >0$. This results in the effective vanishing of the term $(1 - \alpha \lambda_{k+1})^t$ in the above expression after a sufficient number of iterations and $c_i$'s being bounded away from zero. According to Theorem \ref{Thm: gen}, we will then have an approximate linear combination of the $k$ characteristic vectors of the graph i.e., $\| M^t \mathbf{x}^0 - \sum_{i=1}^k c_i \chi_i \| $ will be small. Small perturbation assumption also leads to $\|\bm{\tilde{\phi}}_i \|$ being relatively small. Note that this eigengap condition is equivalent to Assumption A1 in \cite{Ng2002}. \par

It is worth noting that as the number of clusters $k$ grows, it becomes increasingly difficult to distinguish the clusters from the vector $M^t \mathbf{x}^0$ using the classical $k$-means algorithm because the perturbation $\bm{\tilde{\phi}}_i$  will accompany the $k$ eigenvectors in $M^t \mathbf{x}^0$. Thus, we devise a recursive bi-partitioning mechanism to find the clusters.

\subsection{Clustering algorithm}
\label{Sec: Stop}
Our analysis in the previous section suggests that points in the same cluster mix quickly whereas points in different clusters mix slowly. Simon and Ando's \cite{Simon1961} theory of nearly completely decomposable systems also demonstrates that states in the same subsystem achieve local equilibria long before the system as a whole attains a global equilibrium. Therefore, an efficient clustering algorithm should stop when a local equilibrium is achieved. We can then distinguish the clusters based on mixing of the points. The two clusters in this case correspond to aggregation of elements of $\mathbf{x}^t$. Thus a simple search for the largest \emph{gap} in the sorted $\mathbf{x}^t$ can reveal the clusters. This cluster separating \emph{gap} is directly proportional to $b$, since we initialize $\mathbf{x}^0$ by choosing $n$ points uniformly at random from an interval $[0,b]$. Furthermore, it is inversely proportional to the size $n$ of the dataset. Thus we define the \emph{gap} between two consecutive elements of sorted $\mathbf{x}^t$ as:
\begin{equation}
gap(i) = \left \{ \begin{array}{c@{\qquad} l}
x^t_{i+1} - x^t_i  &  \text{if~ } x^t_{i+1} - x^t_i \ge \frac{b}{2n},\\
0 & \text{otherwise}.
\end{array} \right. \label{Eq: Gap}
\end{equation}
In each recursive call, the algorithm terminates upon finding the largest \emph{gap} and bi-partitions the data based on this \emph{gap}. If the algorithm fails to find a nonzero \emph{gap} in a recursive call, then the indexing set of $\mathbf{x}$ in this call corresponds a cluster. This leads us to Algorithm \ref{Alg: agents-r1}  (RARD - Theoretical). \par

\begin{algorithm}
\caption{Recursive Agent-Based Resource Diffusion (RARD) - Theoretical}\label{Alg: agents-r1}
\begin{algorithmic}[1]
\State{\textbf{Input:} Matrix $M = (1 - \alpha)I + \alpha D^{-1} W$ and a tolerance $\epsilon$}
\Procedure{$C = $\,RARD}{$M,~ \epsilon$}
\State{$n \gets rowsize(M)$}
\State{Initialized $\mathbf{x}^0$ by choosing $n$ points uniformly at random from $[0,b]$.}
\Repeat
\State $\mathbf{x}^{t+1} \gets M\mathbf{x}^t $
 \Until{$(1 - \alpha \lambda_{k+1})^t \le \epsilon$}
\State{Sort$(\mathbf{x}^{t+1})$; find the largest \emph{gap} using Equation	\eqref{Eq: Gap}.}
\If {\emph{gap} is not found}
\State \textbf{return} 
\EndIf
\State{bi-partition the indexing set of $\mathbf{x}^{t+1}$ based on largest \emph{gap} into $i_1$ and $i_2$.}
\State{$C\gets\textsc{RARD}(M(i_1,i_1), ~\epsilon)$,\qquad $C\gets\textsc{RARD}(M(i_2,i_2), ~\epsilon)$}
\EndProcedure
\State{\textbf{Output:} Clustering $C$.}
\end{algorithmic}
\end{algorithm}

In practice, we do not know the eigenvalues of $M$. Thus, in our implementation we start the procedure with an initial tolerance $\epsilon_0$ on mixing of $\mathbf{x}^t$. When the tolerance is achieved, we search for a nonzero \emph{gap} in the vector sort($\mathbf{x}^{t+1}$). If a \emph{gap} is found the dataset is bi-partitioned based on the largest \emph{gap}. In a recursive fashion, the bi-partitioning procedure is then applied to both resulting partitions. On the other hand, if the \emph{gap} is not found we decrease the tolerance and reevaluate for a \emph{gap} after the new tolerance is attained. A cluster is formed if the procedure can not find a \emph{gap} using either a maximum number of iterations $t_{\max}$ or a minimum tolerance $\epsilon_{\min}$. This algorithm is Algorithm \ref{Alg: agents-r} (RARD - Implemented). Since our algorithm finds clusters in a recursive fashion by bi-partitioning the dataset in each recursive step, Theorem \ref{Thm: gen} can be reduced to the following corollary.
\begin{cor}
Suppose the dataset contains 2 clusters and let $\mathbf{x}^0$ be any vector such that each $x_i^0>0$ and $(\mathbf{x}^0)^T \mathbf{1} = 1$, then we have 
\[
\| M^t \mathbf{x}^0 - \sum_{i=1}^2 c_i \bm{\chi}_i \|   \le   \left( \sum_{i=1}^2 2\| \bm{\tilde{\phi}}_i \| +  \|\bm{\tilde{\phi}}_i \|^2   + (1 - \alpha \lambda_{3})^t \right) \frac{\max_j\sqrt{d_j}}{\min_j \sqrt{d_j}},
\]
where $c_i =  (1 - \alpha \tilde{\lambda}_i)^{t} \frac{\bm{\chi}_i^T D \mathbf{x}^0 }{ \mathbf{1}^TD \bm{\chi}_i}$ and $d_j$ is the degree of the $j^{th}$ node.
\end{cor}

\begin{algorithm}
\caption{Recursive Agent-Based Resource Diffusion (RARD) - Implemented}\label{Alg: agents-r}
\begin{algorithmic}[1]
\State{\textbf{Input:} Matrix $M = (1 - \alpha)I + \alpha D^{-1} W$ and initial tolerance $\epsilon_0$}
\Procedure{$C = $\,RARD}{$M, ~\epsilon_0$}
\State{$n \gets rowsize(M)$}
\State{Initialized $\mathbf{x}^0$ by choosing $n$ points uniformly at random from $[0,b]$.}
\State{Initialize $\epsilon \gets \epsilon_0$}
\Repeat
\Repeat
\State $\mathbf{x}^{t+1} \gets M\mathbf{x}^t , ~~~~~~  y^{t+1} \gets \|\mathbf{x}^{t+1} - \mathbf{x}^t\| $
\Until{ $| y^{t+1} - y^t| \le \epsilon$}
\State{Sort$(\mathbf{x}^{t+1})$; find the largest \emph{gap} using Equation	\eqref{Eq: Gap}.}
\If {$\epsilon \le \epsilon_{\min} \textbf{ or } t \ge t_{\max}$}
\State \textbf{return} 
\EndIf
\State $\epsilon \gets \epsilon/2$
\Until{A nonzero \emph{gap} is found}
\State{bi-partition the indexing set of $\mathbf{x}^{t+1}$ based on largest \emph{gap} into $i_1$ and $i_2$.}
\State{$C\gets\textsc{RARD}(M(i_1,i_1), ~\epsilon_0)$,\qquad $C\gets\textsc{RARD}(M(i_2,i_2), ~\epsilon_0)$}
\EndProcedure
\State{\textbf{Output:} Clustering $C$.}
\end{algorithmic}
\end{algorithm}



\subsection{Time complexity}
Each iteration of Algorithm \ref{Alg: points} involves multiplication of two matrices of size $n \times n$ and $n \times d$, which requires $O(\bar{n}d)$ operations, where $\bar{n}$ is the number of nonzero entries in the matrix $W$. Algorithm \ref{Alg: points} also calls the $k$-means algorithm, whose running time for one iteration is $O(nkd)$. So the complexity of Algorithm \ref{Alg: points} is $O(\bar{n}d t_{\max}) + O(nkd)$, where $t_{\max}$ is the maximum number of iterations. 

\begin{figure}[!ht]
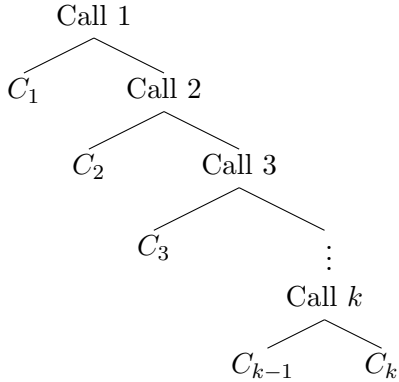

\Tree[.{Call 1} [.$C_1$ ]
          [.{Call 2} [.$C_2$ ]
                [.{Call 3} [.$C_3$ ]
                          [.{~~~~\vdots~~~~\\ Call $k$}  $~~~C_{k-1}~~~~$
                                      $C_{k}$ ]]]]
\caption{A depiction of the worst case.}
\label{Fig: worst}
\end{figure}

Algorithm \ref{Alg: agents-r} is a recursive algorithm. For the sake of intuitive analysis, we assume that $n = 2^j$ for $j\ge1$ and that all clusters are of equal size i.e., each cluster has $n/k$ points. we further assume that $k = 2^\ell$ for $\ell \ge 1$. In each recursive call, the algorithm performs a maximum of $t_{\max}$ sparse matrix-vector multiplications which require $O(\bar{n}t_{\max})$ operations. It also makes two more recursive calls and sorts $n$ numbers except for the base case. Thus each non-base call takes $O(\bar{n}t_{\max})+O(n\log n)$ operations, where $O(n \log n)$ is the complexity of sorting $n$ numbers. Let $T(n,k,t_{\max})$ be the time required to find $k$ clusters in a dataset of size $n$ by Algorithm \ref{Alg: agents-r}, then we have
\begin{align*}
T(n,k,t_{\max}) &= 2T\left( \frac{n}{2},k,t_{\max} \right) + c \bar{n} t_{\max} + n\log n  \\
& = 4T\left( \frac{n}{4},k,t_{\max} \right) + c \bar{n} t_{\max} + c \frac{\bar{n}}{2} t_{\max} +  n\log n + \frac{n}{2}\log \frac{n}{2}  \\
&\quad \vdots \\
& = 2^\ell T\left( \frac{n}{2^\ell} ,k, t_{\max} \right) + c \bar{n} t_{\max} \sum_{s=0}^{\ell - 1} \frac{1}{2^s} + n\sum_{s=0}^{\ell - 1} \frac{1}{2^s}\log \frac{n}{2^s} \\
& = 2^\ell c \frac{\bar{n}}{2^\ell} t_{\max} + c \bar{n} t_{\max} \sum_{s=0}^{\ell - 1} \frac{1}{2^s} + n \log n \sum_{s=0}^{\ell - 1} \frac{1}{2^s}  - n \sum_{s=0}^{\ell - 1} \frac{s}{2^s} \\
& =  c \bar{n} t_{\max} + (c \bar{n} t_{\max} + n \log n) \sum_{s=0}^{\ell - 1} \frac{1}{2^s} - n \sum_{s=0}^{\ell - 1} \frac{s}{2^s} \\
& = O(\bar{n} t_{\max}) + O(\bar{n} t_{\max}) + O(n \log n) - O(n)  \\
&= O(\bar{n} t_{\max}) + O(n \log n),
\end{align*}
where $c$ is a constant and we have used the fact that $\sum_{s=0}^{\ell} 1/2^s \le 2$ and $\sum_{s=0}^{\ell} s/2^s \le 2$. Similarly, if we assume a different split of the data in each recursive call, such as $1/3$ and $2/3$, or $1/4$ and $3/4$ etc., it easily follows from the above analysis that the running time of the Algorithm \ref{Alg: agents-r} remains $O(\bar{n}t_{\max})+ O(n \log n)$. The worst case arises when the dataset comprises a big cluster $\mathcal{V}_k$ of size (say) $n/2$ and rest of the dataset constitutes the other $k-1$ clusters $\mathcal{V}_1, \mathcal{V}_2, \ldots \mathcal{V}_{k-1}$. In such a scenario, if a smaller cluster from one of the $k-1$ clusters, say $\mathcal{V}_1$, splits from the data in the first recursive call, and in the second call on the dataset containing the big cluster, another smaller cluster $\mathcal{V}_2$ splits from the data and so on. Continuing in this way, suppose that the big cluster is the only cluster left in the last recursive call as shown in Figure \ref{Fig: worst}, then we have made $k$ recursive calls on a dataset of size at least $n/2$ making the running time of the algorithm $O(\bar{n}k t_{\max})+ O(n k\log n)$. However, if $\mathcal{V}_k$ splits from the dataset early on in the recursion, the running time remains $O(\bar{n} t_{\max})+ O(n \log n)$. Thus, unless the dataset contains a big cluster encompassing a dominant fraction of the dataset, the running time of the Algorithm \ref{Alg: agents-r} is $O(\bar{n} t_{\max})+ O(n \log n)$. We use $p$ nearest neighbors to compute the similarities between points, which results in $O(pn)$ nonzero entries in $W$. Since $p$ is a constant typically between 4 and 10, we conclude that number of nonzero entries in $W$ is $O(n)$. Thus for the $p$-nearest neighbor similarity function, the complexity of Algorithm \ref{Alg: agents-r} is $O(n t_{\max})+ O(n \log n)$.


\section{Simulation results}
We begin with a toy example to illustrate the mechanics of Algorithm \ref{Alg: agents-r}. Suppose that we have the  following normalized similarity matrix for a dataset with three clusters. The intra-cluster similarities of the three clusters are represented by red, green and blue colors. 

\[
 \left[
\begin{array}{c@{~~} c@{~~} c @{~~} c @{~~} c @{~~} c @{~~} c @{~~} c @{~~} c @{~~} c }
\color{red}0		&	\color{red}.5	& \color{red}.45	& .025	& .025	&	0	&	0	&	0	&	0	&	0	\\
\color{red}.4		&	\color{red}0	& \color{red}.55	&	0	&	0	&	0	& .05	&	0	&	0	&	0   \\
\color{red}.3		&	\color{red}.7	&	\color{red}0	&	0	&	0	&	0	&	0	&	0	&	0	&	0	\\
0		&	.01	&	0	&	\color{ForestGreen}0	&\color{ForestGreen}	.3	&	\color{ForestGreen}.4	&	\color{ForestGreen}.28	&	0	&	.01	&	0	\\
0		&	0	&	0	&	\color{ForestGreen}.4	&	\color{ForestGreen}0	&	\color{ForestGreen}.3	&	\color{ForestGreen}.3	&	0	&	0	&	0	\\
0		&  	0	&	.1	&	\color{ForestGreen}.25	&	\color{ForestGreen}.25	&	\color{ForestGreen}0	&	\color{ForestGreen}.4	&	0	&	0	&	0	\\
.01		&	0	&	0	&	\color{ForestGreen}.4	&	\color{ForestGreen}.3	&	\color{ForestGreen}.27	&	\color{ForestGreen}0	&.02	&	0	&	0	\\
0		&	.01	&	0	&	0	&	0	&	0	&	0	&	\color{blue}0	&	\color{blue}0.5	&	\color{blue}0.49	\\
0		&	0	&	0	&	.02	&	0	&	0	&	0	&	\color{blue}0.49	&	\color{blue}00	&	\color{blue}0.49	\\
0		&	0	&	0	&	0	&	0	&	0	&	0	&	\color{blue}0.7	&	\color{blue}0.3	&	\color{blue}00
\end{array} \right]
\]

The first call to the RARD procedure in Algorithm \ref{Alg: agents-r} starts by picking 10 points uniformly at random from $[0,100]$. Stopping criterion is met after 27 iterations giving us the following vector. A bipartition of the dataset based on the \emph{gap} criteria separates the blue cluster from the other two clusters. 

\[
\left [
\begin{array}{c}
\color{red}31.92	\\
\color{red}32.15	\\
\color{red}31.92	\\
\color{ForestGreen}32.64	\\
\color{ForestGreen}32.60	\\
\color{ForestGreen}32.55	\\
\color{ForestGreen}32.71	\\
\color{blue}39.57	\\
\color{blue}39.53	\\
\color{blue}39.63
\end{array}
\right] 
\] 

Algorithm \ref{Alg: agents-r} then makes two recursive calls to the RARD procedure on the re-normalized sub-matrices corresponding to the two partitions. The recursive call on the first partition containing blue and green clusters results in a bipartition of the data into two clusters giving us the following vector, whereas the recursive call on the second partition containing the blue cluster does not find any \emph{gap} satisfying the stopping criteria, so the algorithm identifies it as a cluster and returns out of the recursion.  

\[
\left[
\begin{array}{c}
\color{red}77.63 	\\
\color{red}78.03 	\\
\color{red}78.07 	\\
\color{ForestGreen}67.06 	\\
\color{ForestGreen}66.91 	\\
\color{ForestGreen}67.78 	\\
\color{ForestGreen}66.98 	\\  
\end{array} \right]
\left[
\begin{array}{c}
 \color{blue}  21.12 \\
 \color{blue}   20.98\\
 \color{blue}   21.28
\end{array}
\right]
\] 

The recursive call on the red and green clusters, then makes two further calls to the RARD procedure, one on the red and one on the green cluster. These two calls do not find a bipartition in the data and return out of the recursion identifying red and green clusters.

In the following sections, we demonstrate the efficiency of the proposed algorithm on a variety of synthetic and real datasets. In particular, we show that Algorithm \ref{Alg: agents-r} can identify clusters of complex shapes and varying sizes in Section \ref{Sec: Syn} and compare its accuracy with the normalized cut algorithm \cite{Shi2000}. We have used $p$-neighbor similarity measure to construct the similarity graph for all the experiments. In Section \ref{Sec: Scale}, we portray the scalability and speed of Algorithm \ref{Alg: agents-r} by applying it to large-scale stochastic block models. Finally in Section \ref{Sec: Real}, we run Algorithm \ref{Alg: agents-r} on two real datasets and compare its accuracy and speed with normalized cut algorithm, fast approximate spectral clustering \cite{Yan2009} and Nystrom method \cite{Fowlkes2004}. We implement all the algorithms in MATLAB 8.4.0 and conduct experiments on a machine with Intel Core i7 3.40GHz CPU and 16GB memory.


\subsection{Performance evaluation}
Mutual information is a symmetric measure to quantify the information shared between two distributions. It is widely used as a measure to calculate the shared information between two clusterings. Let $\mathscr{V}$ denote the cluster labels and $\mathscr{V}^\prime$ be the clustering obtained by an algorithm. Their mutual information is defined as follows:
\[
MI(\mathscr{V}, \mathscr{V}^\prime) = \sum_{\mathcal{V}_i \in \mathscr{V}, ~ \mathcal{V}_j^\prime \in \mathscr{V}^\prime} Pr(\mathcal{V}_i, \mathcal{V}_j^\prime) ~\log\left( \frac{Pr(\mathcal{V}_i, \mathcal{V}_j^\prime)}{Pr(\mathcal{V}_i) Pr( \mathcal{V}_j^\prime)}	 \right)
\]
where $Pr(\mathcal{V}_i)$ and $Pr(\mathcal{V}_j^\prime)$ are the probabilities that an arbitrary point belongs to cluster $\mathcal{V}_i$ in clustering $\mathscr{V}$ and $\mathcal{V}_j^\prime$ in clustering $\mathscr{V^\prime}$  respectively, i.e.,
\[
Pr(\mathcal{V}_i) = \frac{|\mathcal{V}_i|}{n} \text{ ~and~ } Pr(\mathcal{V}_j^\prime) = \frac{| \mathcal{V}_j^\prime|}{n}
\]
 $Pr(\mathcal{V}_i, \mathcal{V}_j^\prime)$ is the joint probability that an arbitrary point lies in both clusters $\mathcal{V}_i$ and $\mathcal{V}_j^\prime$ in clusterings $\mathscr{V}$ and $\mathscr{V^\prime}$ respectively, i.e.,
\[
Pr(\mathcal{V}_i, \mathcal{V}_j^\prime) = \frac{|\mathcal{V}_i \cap \mathcal{V}_j^\prime|}{n}
\]
For the ease of interpretation, we use normalized mutual information defined as:
\[
NMI(\mathscr{V}, \mathscr{V}^\prime) = \frac{MI(\mathscr{V}, \mathscr{V}^\prime)}{\sqrt{H(\mathscr{V})\, H( \mathscr{V}^\prime)}}
\]
where $H(\mathscr{V})$ is the entropy of $\mathscr{V}$ given by:
\[
H(\mathscr{V}) = -\sum_{\mathcal{V}_i \in \mathscr{V}} Pr(\mathcal{V}_i) \log (Pr(\mathcal{V}_i))
\]
It is easy to verify that $0 \le NMI(\mathscr{V}, \mathscr{V}^\prime) \le 1$. NMI is 1 when the two clusterings are identical and 0 when the clusterings are independent.

\begin{figure}[!t]
     \centering
     \subfloat[][Mixture of Gaussians]{\includegraphics[width=.25\linewidth]{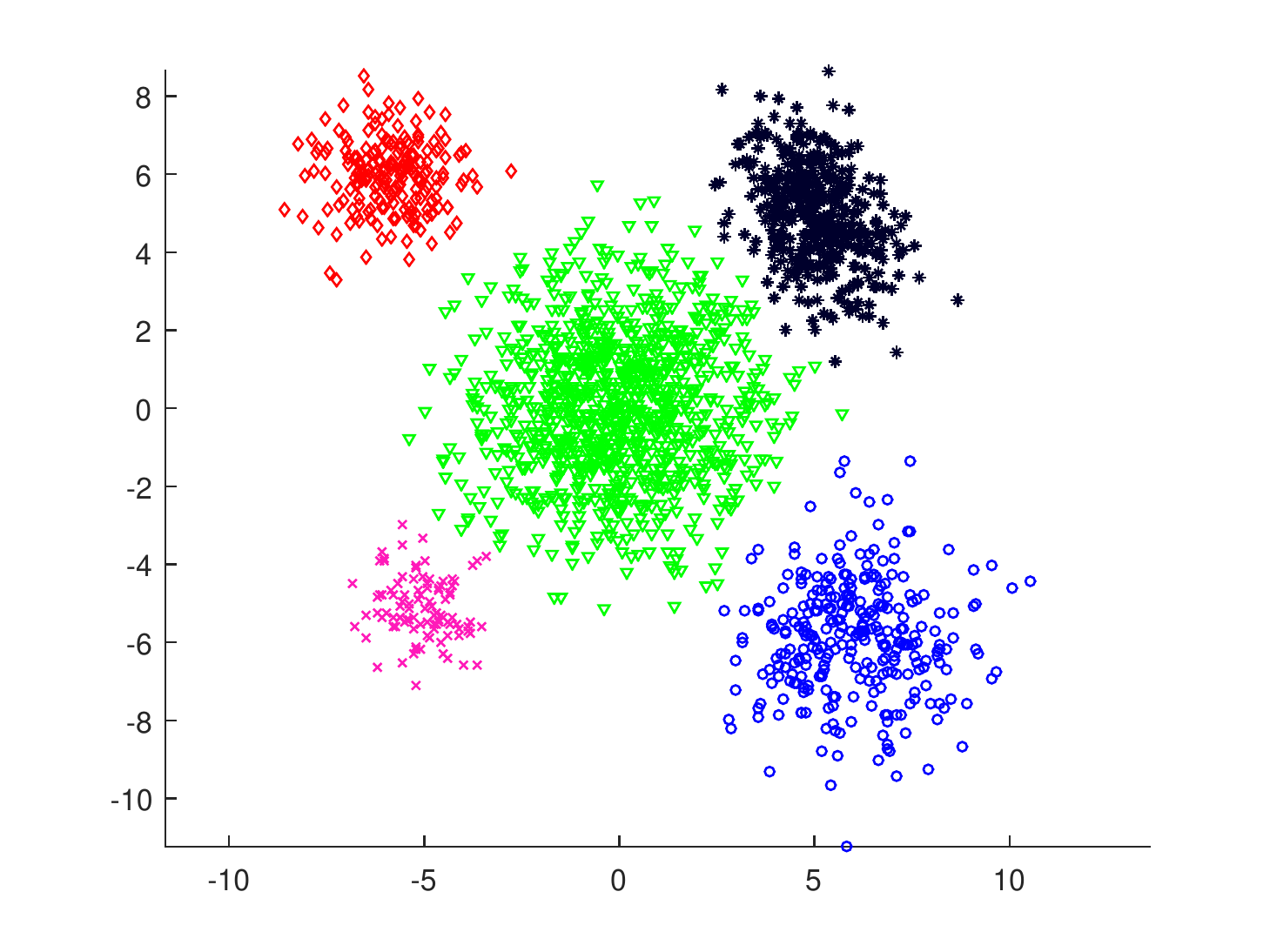}\label{Fig: alg1-1}}
     \subfloat[][Clustering aggregation]{\includegraphics[width=.25\linewidth]{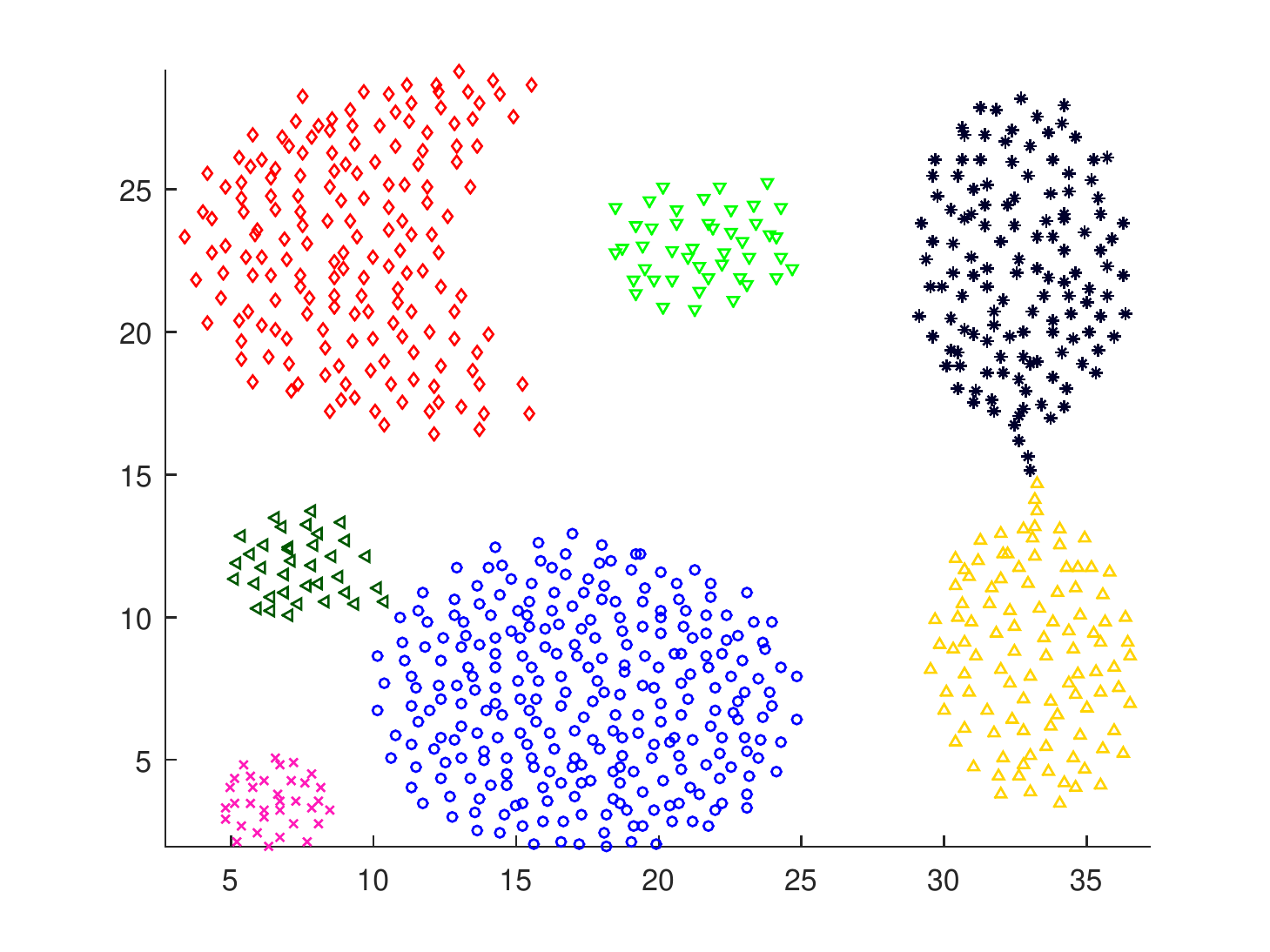}\label{Fig: alg1-2}}
     \subfloat[][Two crescents]{\includegraphics[width=.25\linewidth]{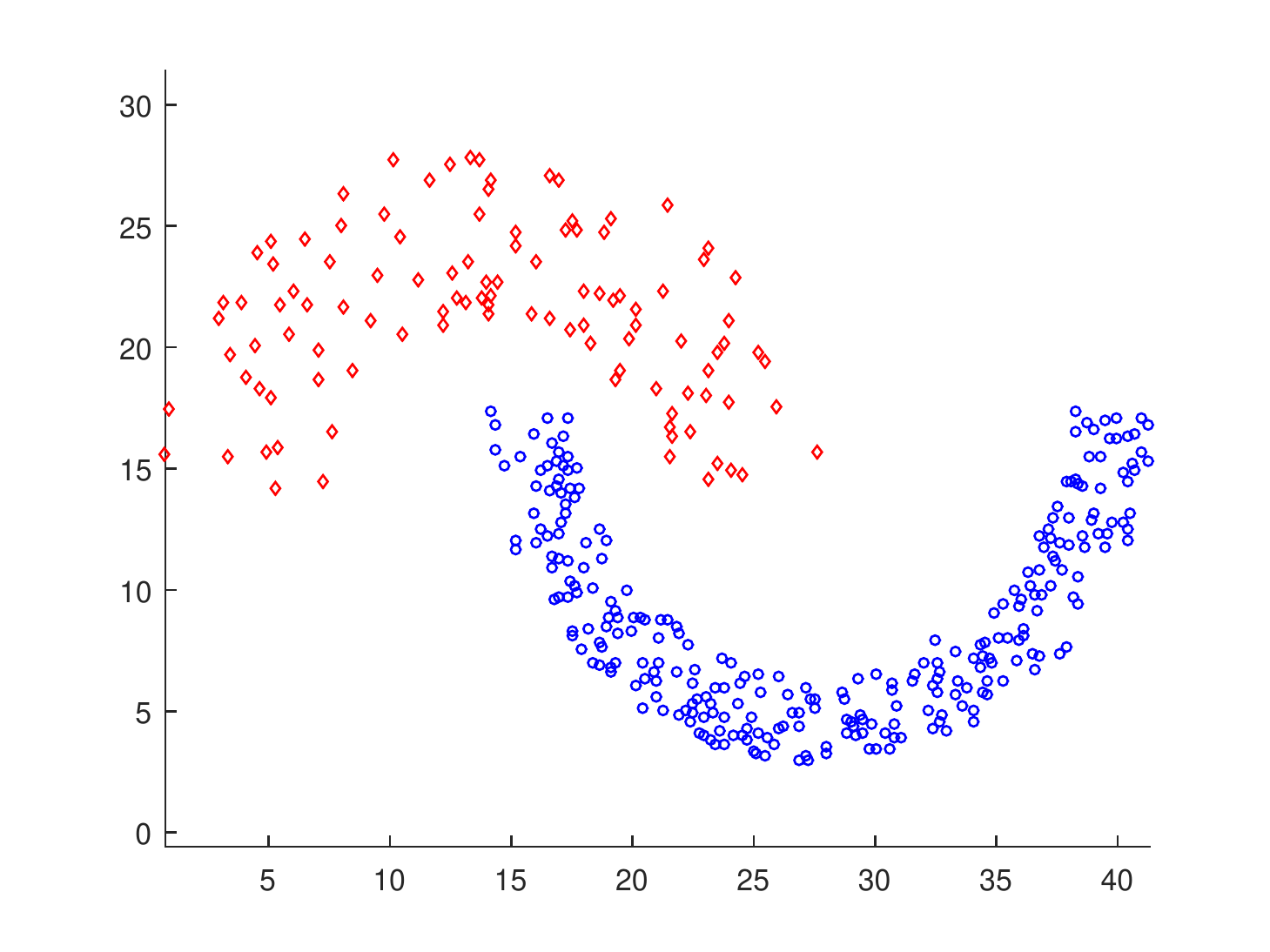}\label{Fig: alg1-3}}
     \subfloat[][Half ellipses]{\includegraphics[width=.25\linewidth]{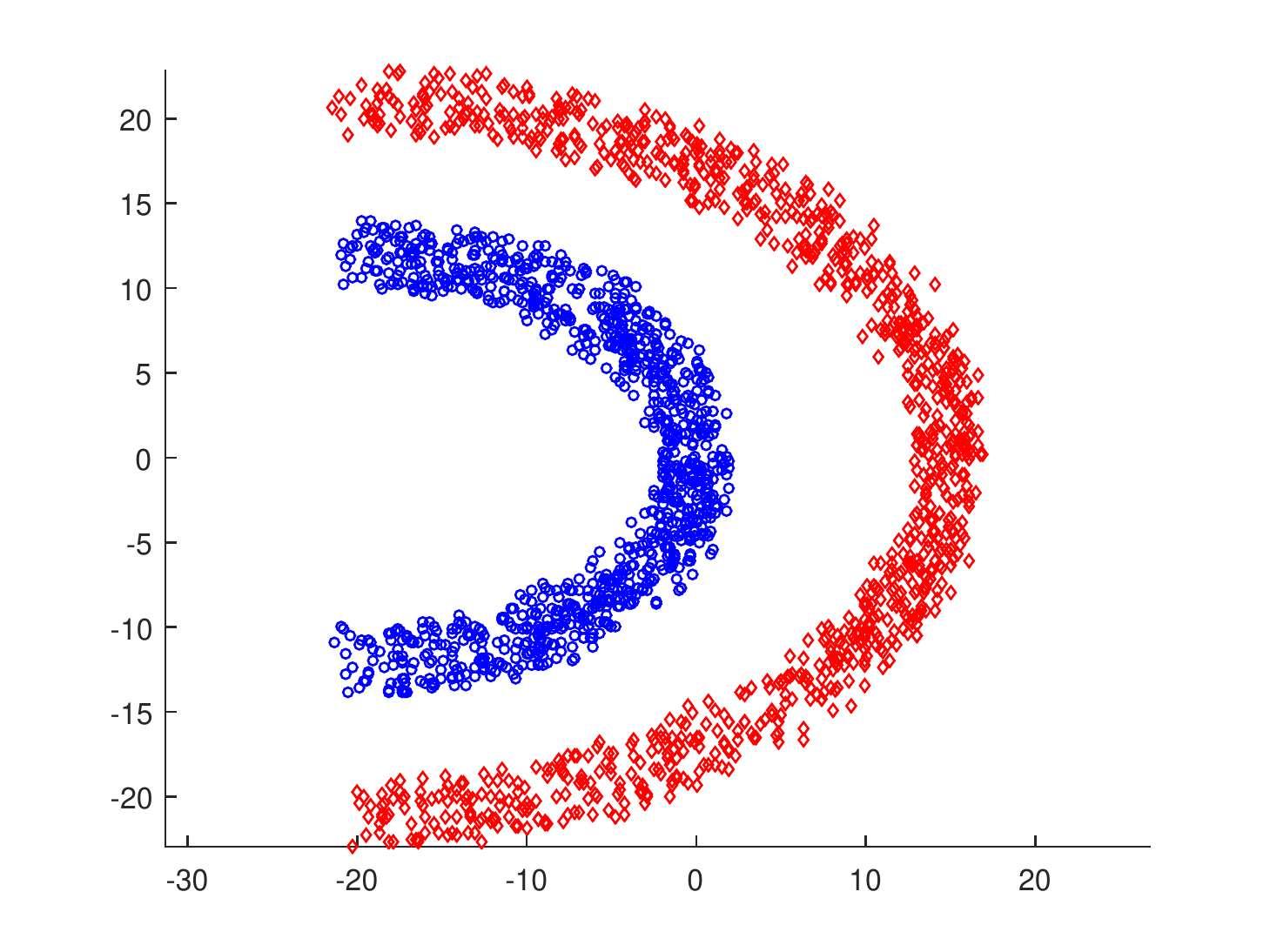}\label{Fig: alg1-4}}               
     \par
     \caption{Clustering result of Algorithm \ref{Alg: agents-r} on synthetic datasets.}
     \label{Fig: alg1}
\end{figure}
\begin{table}[!t]
\caption{Results of 50 runs of Algorithm \ref{Alg: agents-r} on synthetic datasets. Parameter settings $\alpha = 1$ and $\epsilon_0 = 10^{-2}, ~10^{-3}, ~10^{-3}, ~10^{-4}$ top to bottom. NCut shows the average normalized cut defined in Equation \eqref{Eq: Ncut} over 50 simulations with standard error. NMI is the mean normalized mutual information in 50 runs with standard error. Runtime shows the average computation time of a simulation. }
\label{Tab: algo2_r}
\begin{center}
\begin{tabular}{l c c c }
\toprule
{\bf Datasets}  &{\bf NCut} &{\bf NMI}  &{\bf Runtime} \\ 
\toprule 
Mixture of Gaussians & $0.0052\pm 0.0000$ & $97.12\pm 0.03$  & 0.687 \\
\midrule
Clustering aggregation &  $0.0323\pm 0.0000$ & $99.58\pm 0.03$ & 0.106 \\
\midrule
Two crescents & $0.0052\pm0$ & $100\pm 0$& 0.092 \\
\midrule
Half ellipses &  $0\pm 0$ & $100\pm 0$ & 0.969 \\ 
\hline 
\end{tabular}
\end{center}
\end{table}

\subsection{Synthetic datasets}
\label{Sec: Syn}
Four two-dimensional synthetic datasets have been used to compare Algorithm \ref{Alg: agents-r} to the normalized cut algorithm \cite{Shi2000}. The details of the datasets are given in Appendix \ref{Ap: Exam}. Table \ref{Tab: algo2_r} portrays the results of RARD algorithm. These results depict that our recursive implementation is very accurate in identifying clusters of complex shapes and different sizes. Small standard errors emphasize that RARD algorithm has little dependence on the initial vector $\mathbf{x}^0$. 

\begin{table}[t!]
\caption{Computation time (in seconds) of Algorithm \ref{Alg: agents-r} on SBMs with $p=0.5$ and $q = 0.01$. The time shown is averaged over 50 simulations on different SBMs. All partitions are exactly recovered. }
\label{Tab: stock}
\begin{center}
\begin{tabular}{c c c c  c  c c c c c c c c}
\toprule
$n$  & \multicolumn{3}{c}{$15,000$} & \multicolumn{3}{c}{$30,000$} & \multicolumn{3}{c}{$60,000$}  \\
\midrule
{$k$}  & 5 & 10 & 15 & 5 & 10 & 15 & 5 & 10 & 15 \\
\midrule
\textbf{RARD}  &  4.19 &   3.51 &   3.23 & 14.47 & 11.59 &  10.46  & 78.21 &  53.27 &  45.44  \\
\midrule
\textbf{Normalized Cut}  &   69.46  &  121.14 &    250.07 & \multicolumn{6}{c}{Out of Memory}  \\
\bottomrule
\end{tabular}
\end{center}
\end{table}

\subsection{Scalability}
\label{Sec: Scale}
We apply Algorithm \ref{Alg: agents-r} to stochastic block model (SBM) graphs to illustrate its scalability to large datasets with many clusters. We also compare the runtime with normalized cut algorithm \cite{Shi2000}. In the basic form, a SBM with same size blocks (clusters) is defined by four parameters; $n$, the number of vertices; $k$, the number of blocks (clusters); $p$, the probability of an edge between two points in the same cluster and $q$, the probability of an edge between two points in different clusters. Running time of the algorithm on various size SBMs is shown in Table \ref{Tab: stock}. Runtimes shown are average times for 50 different SBMs. Observe that RARD algorithm is significantly faster than normalized cut. It also consumes less memory as compared to normalized cut. For each model, the RARD algorithm recovers all the clusters exactly in each run with $p = 0.5$ and $q =0.01$. Each node shares roughly $pn/k$ edges within the cluster and $q(n - n/k)$ edges across the cluster. For example with $n = 30,000$ and $k = 10$, a node approximately has an edge with $1500$ nodes in its cluster and $270$ edges with nodes in other clusters. Total edges in this graph are roughly $0.5(1770 \times 30,000) = 26.55$ million. We also show the ability of the algorithm to recover correct clusters as we increase the number of edges across clusters. Figure \ref{Fig: s-block} shows the number of simulations where RARD recovered correct clusters in 50 different SBMs while varying $q$, the probability of placing an edge between two nodes in different clusters.

\begin{figure}[t!]
\centering
\includegraphics[width = 0.5\linewidth]{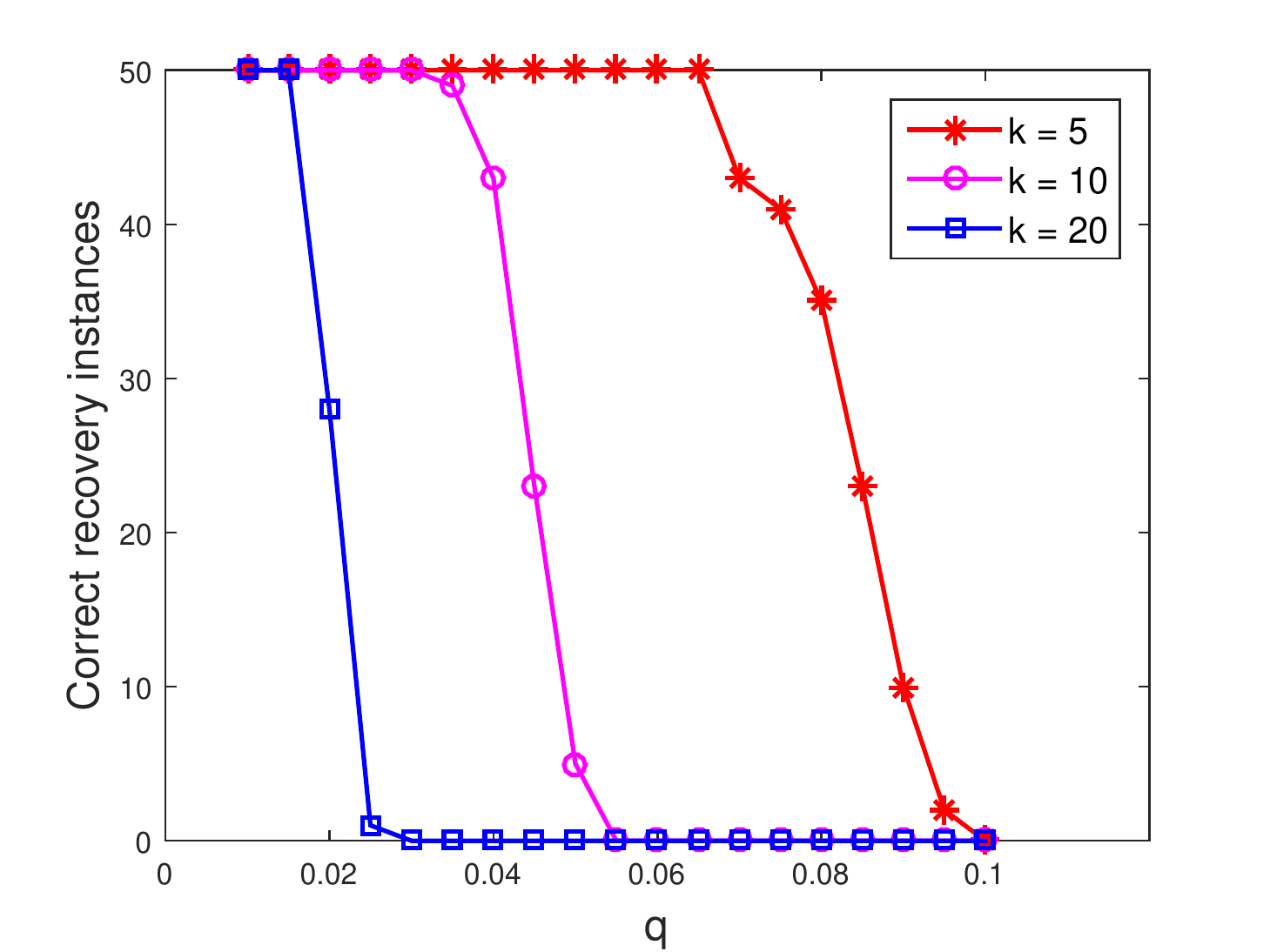}
\caption{Correct clusterings recovered in 50 runs vs. $q$ (probability of an edge between two points in different clusters). Parameters $n=15000$ and $p = 0.5$.}
\label{Fig: s-block}
\end{figure}

\begin{table*}[!t]
\caption{Comparison of Average NMI with standard error and runtime in seconds (in parenthesis) over 50 simulations on real datasets. $k$ denotes the number of clusters. For each $k$, 50 runs are conducted on randomly chosen clusters (except for $k  = 10$ for USPS and $k=20$ for COIL20). N-cut represents the normalized cut algorithm \cite{Shi2000}. FASC is KASP algorithm as defined in \cite{Yan2009}. FASC 1 and 2 are implemented with 10\%  and 5\% representative points respectively. Runtime of FASC is significantly more than our algorithm because it uses k-means to get the representative points which has linear time complexity in terms of dimension of the dataset. NYSTROM 1 and 2 are approximations of Normalized cut algorithm as defined in \cite{Fowlkes2004} with a random sample of 50\% and 20\% respectively. \vspace{-.4cm}}
\label{Tab: usps}
\footnotesize
\begin{center}
\begin{tabular}{c@{~~} c@{~~} c@{~~} c @{~~}c @{~~}c @{~~}c }
\multicolumn{7}{c}{(a) USPS} \\[0.1 cm]
\hline
{$k$}   &{\bf NCUT} & {\bf FASC 1} &  {\bf FASC 2} & {\bf NYSTROM 1} & {\bf NYSTROM 2} & {\bf RARD} 
\\ \hline 
4    &     $90.50\pm2.53(0.48)$     & $84.80\pm1.26(4.46)$ & $84.21\pm1.09 (1.23)$ &  $74.36\pm1.47 (1.57)$   &  $68.66\pm1.53 (0.32)$   &   $88.07\pm1.26(0.37) $         \\
6    &    $85.76\pm0.93(1.13)$   &  $80.88\pm0.87(6.28)$ &  $79.43\pm0.92(2.39)$ &  $73.06\pm1.02 (2.21)$ & $66.46\pm0.69 (0.77)$ &  $85.23\pm0.87(0.58)$     \\
8     &        $85.21\pm0.42(2.31)$   & $79.28\pm0.67(8.93) $ & $78.67\pm0.66(5.31)$ &  $69.93\pm0.63 (4.54)$ & $62.93\pm0.65 (1.88)$  &  $84.56\pm0.41(0.82)$     \\
10    &        $81.47\pm0.00(6.71)$   & $78.18\pm0.42(12.30) $ & $76.95\pm0.51(8.05)$ &  $67.20\pm0.27 (6.60)	$ & $61.94\pm0.30 (1.79)$  &   $82.33\pm0.22(1.14)$     \\
\hline 
\multicolumn{4}{c}{} \\
\multicolumn{7}{c}{(b) COIL20} \\[0.1cm]
\hline 
{$k$}   &{\bf NCUT} &  {\bf FASC 1} & {\bf FASC 2} &{\bf NYSTROM 1} &{\bf NYSTROM 2}  & {\bf RARD}
\\ \hline 
4      &        $   98.68 \pm 0.96 (0.15)      $   & $74.44\pm2.24(0.94)$     & $ 72.82\pm2.04(0.41)$ & $84.86\pm2.05 (0.02)$ & $77.31\pm2.51 (0.01)$ &    $97.32 \pm 1.79 (0.13)$      \\
8      &        $    97.15 \pm 0.90 (0.38)    $   & $73.72\pm1.16(1.30)$  &  $71.35\pm1.27(0.82)$ & $83.82\pm1.15 (0.05)$ &  $73.25\pm1.53 (0.04)$  &  $ 94.03 \pm 0.73 (1.17)$      \\
12    &        $    94.58 \pm 0.79 (0.72) $   & $73.35\pm0.69(2.77)$   & $71.01\pm0.69(1.84)$ & $80.29\pm0.99 (0.11)$ &$69.88\pm1.12 (0.07)$ & $94.36 \pm 0.92 (0.41)$      \\
16   & $    92.28 \pm 0.69 (1.23) $   & $74.46\pm0.61(4.14)$   &  $70.66\pm0.45(2.44)$   & $76.17\pm0.71 (0.20)$ & $67.05\pm0.62 (0.11)$  &  $92.35 \pm 0.28 (0.55)$      \\
20     &        $91.93\pm0.00(1.98)$   & $73.81\pm0.43(6.53)$ & $70.92\pm0.41(3.59)$ & $73.31\pm0.43 (0.32)$  &   $63.61\pm0.46 (0.16)$   & $92.90\pm0.09(0.68)$      \\
\hline 
\end{tabular}
\end{center}
\end{table*}
 

\subsection{Real datasets}
\label{Sec: Real}

We empirically compare the accuracy and speed of our RARD algorithm with normalized cut algorithm \cite{Shi2000}, fast approximate spectral clustering \cite{Yan2009} and Nystrom method \cite{Fowlkes2004} on two real datasets. The USPS dataset has 7291 instances and length of the feature vector is 256 \cite{Le1990}. It has a total of 10 clusters. The COIL20 dataset consists of 1140 examples and has 1024 features with 20 clusters \cite{Nene1996}. We use a p-nearest neighbor graph to construct the similarity matrix. For both dataset we use $p = 4$. $\epsilon_0$ is set to $10^{-3}$ and $10^{-2}$ for USPS and COIL20 datasets respectively. We compare normalized mutual information and computation time of RARD algorithm with other algorithms. As demonstrated earlier RARD algorithm has the same accuracy as normalized cut algorithm. Our algorithm does not sacrifice accuracy as opposed to FASC and Nystrom method which also claim to improve the speed of spectral clustering. For large datasets RARD is faster than both FASC and Nystrom and does not compromise on the accuracy.

\begin{figure}[!t]
     \centering
     \subfloat[][Ncut with 30 segments]{\includegraphics[width=.35\linewidth]{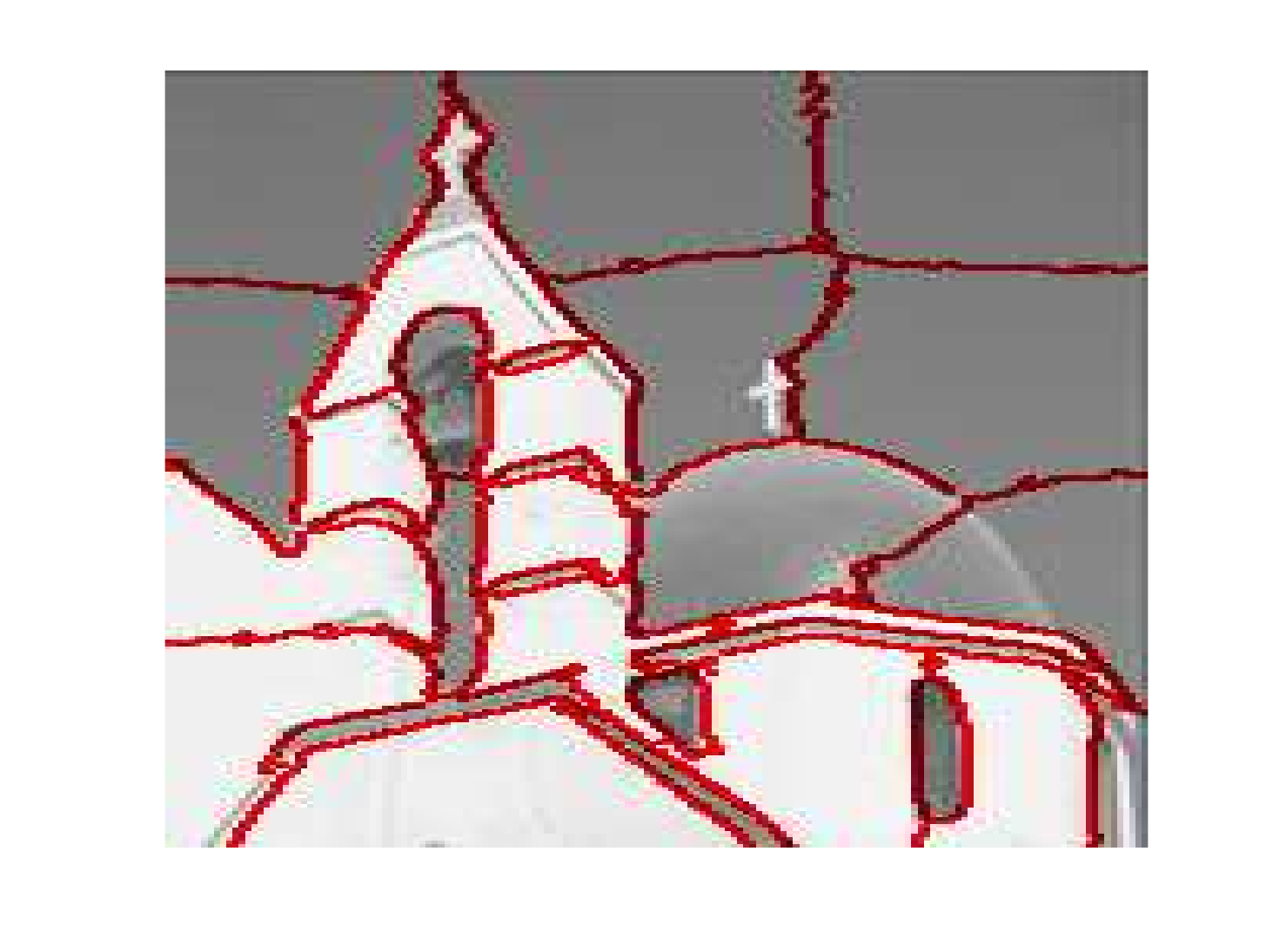}\label{Fig: 7n-30}}
     \subfloat[][RARD with 30 segments]{\includegraphics[width=.35\linewidth]{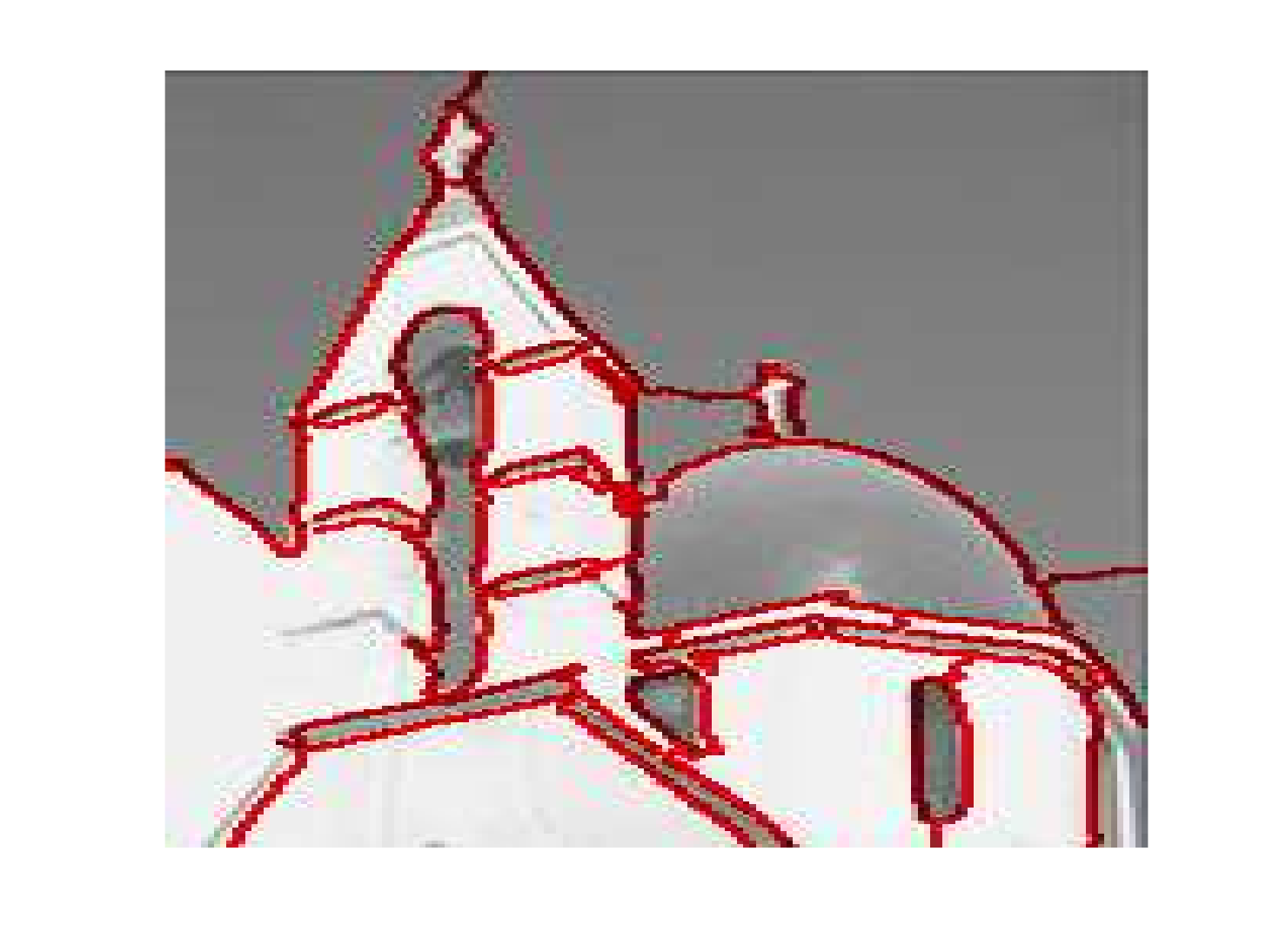}\label{Fig: 7r-33}}
     
     \subfloat[][Ncut with 30 segments]{\includegraphics[width=.35\linewidth]{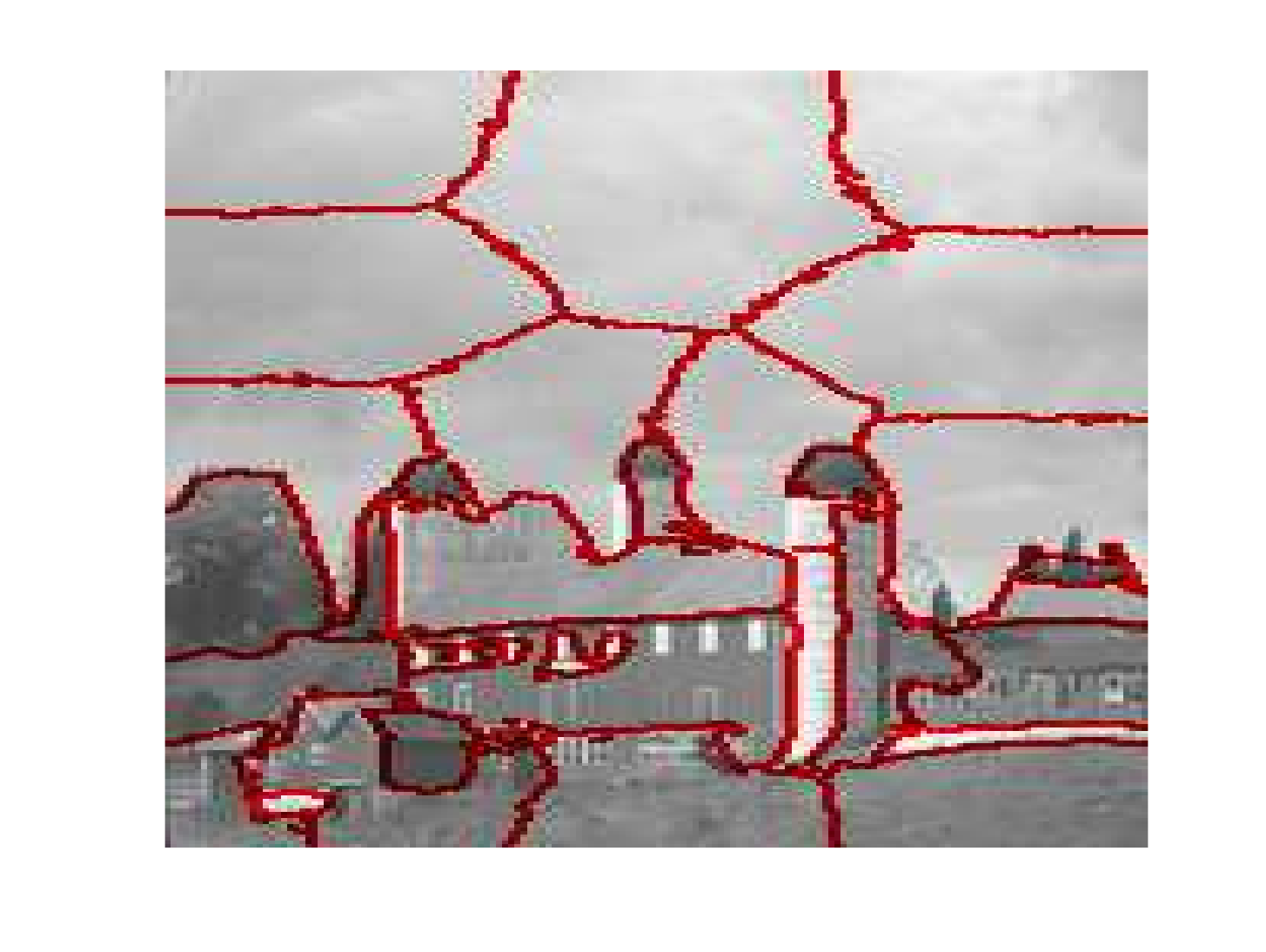}\label{Fig: 11n-30}}
     \subfloat[][RARD with 30 segments]{\includegraphics[width=.35\linewidth]{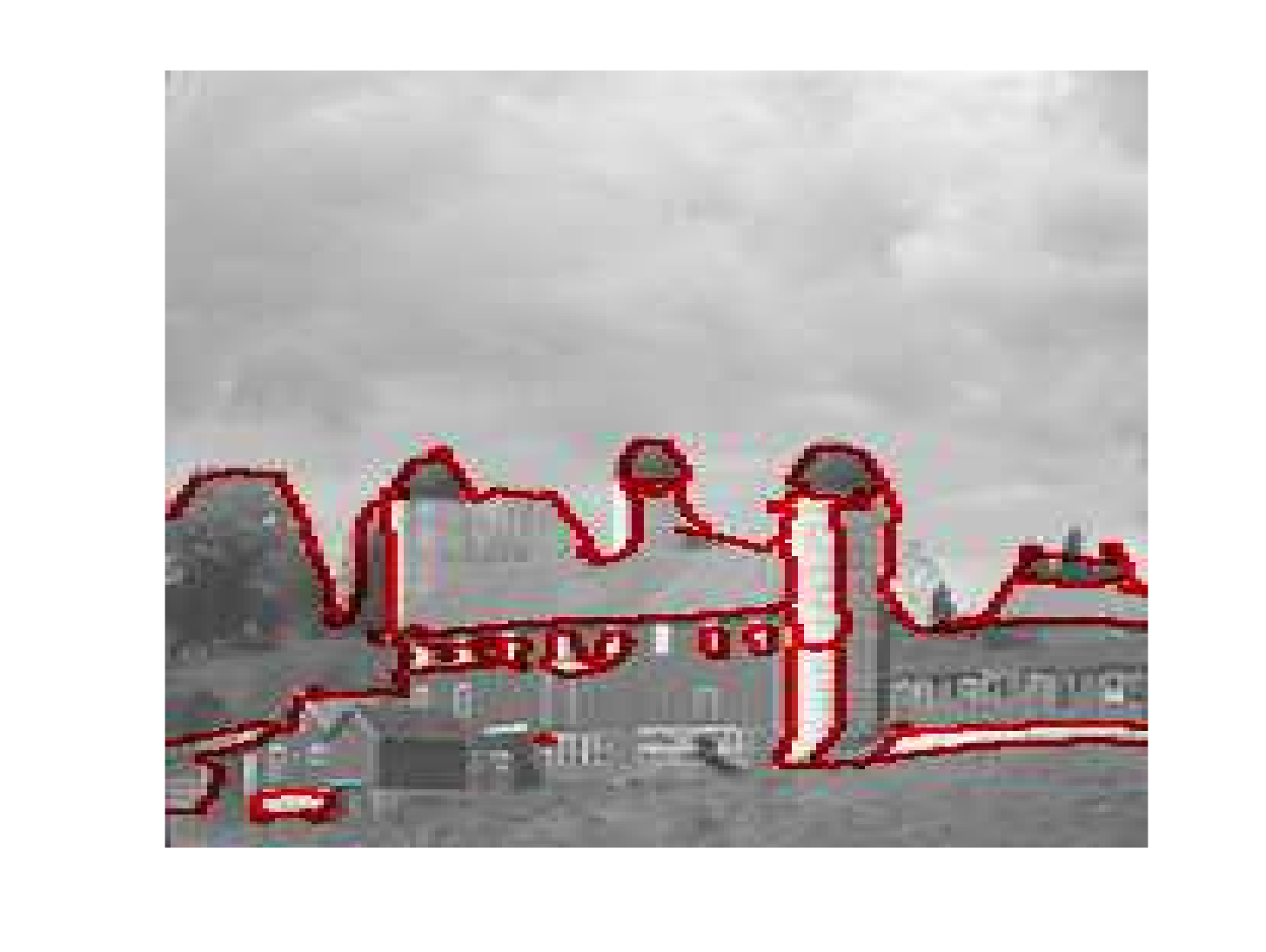}\label{Fig: 11r-27}}               
     \par
     \caption{Comparison of Ncut and RARD on image segmentation.}
     \label{Fig: imgseg}
\end{figure}

\subsection{Image segmentation}
We also test our RARD algorithm against normalized cut algorithm \cite{Shi2000} on image segmentation on some standard images. We show that our algorithm performs significantly better than normalized cut algorithm when the number of segments is large. We follow the feature selection of Shi and Malik \cite{Shi2000}.	In particular, normalized cut fails to detect big segments and thus divides them into smaller sub-segments. On the other hand RARD can detect segments of varying sizes correctly. The resulting segmentations are shown in Figure \ref{Fig: imgseg}. 

\begin{figure}[!t]
     \centering
     \subfloat[][Mixture of Gaussians]{\includegraphics[width=.25\linewidth]{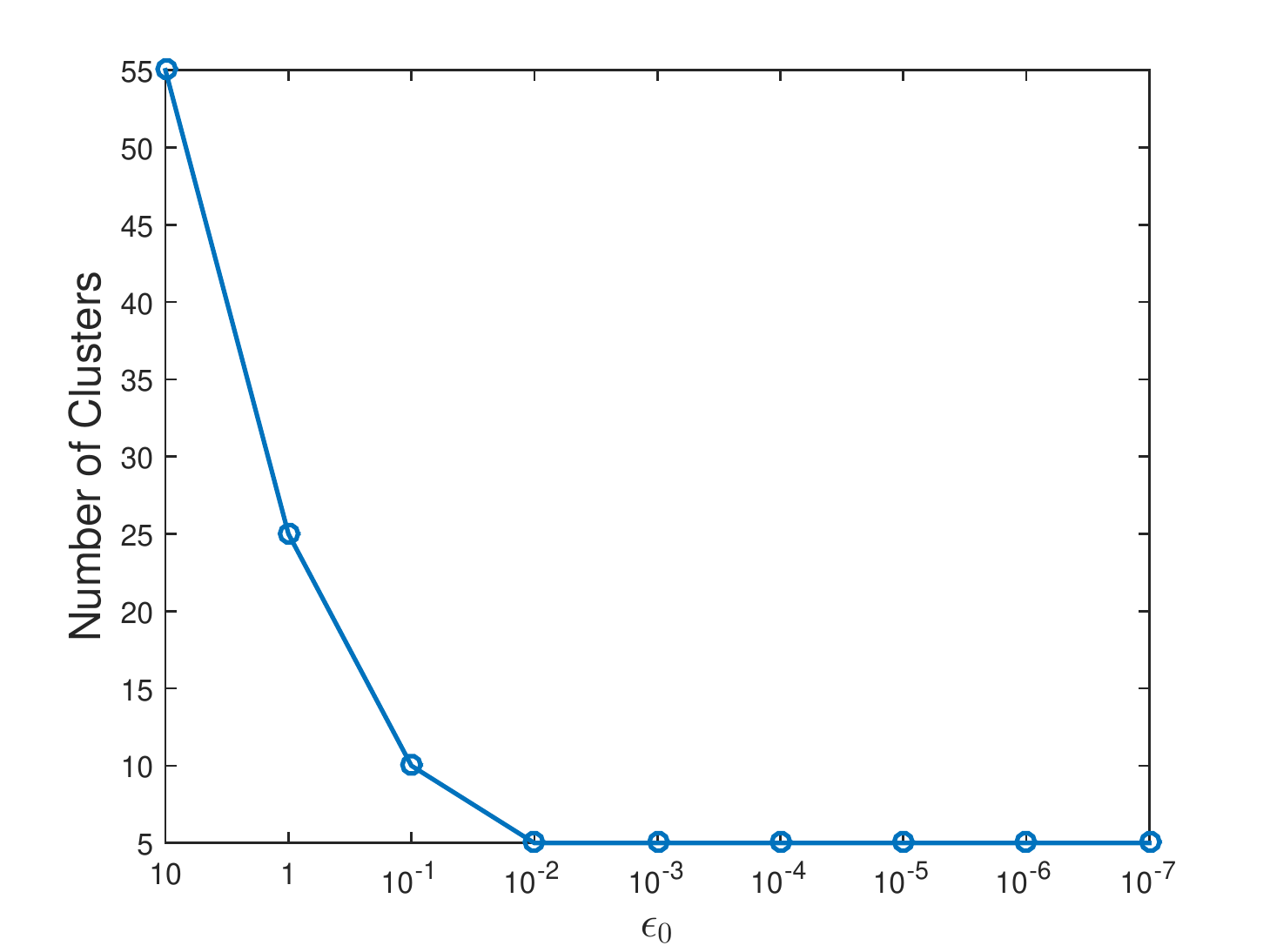}\label{Fig: par-2-1}}
     \subfloat[][Clustering aggregation]{\includegraphics[width=.25\linewidth]{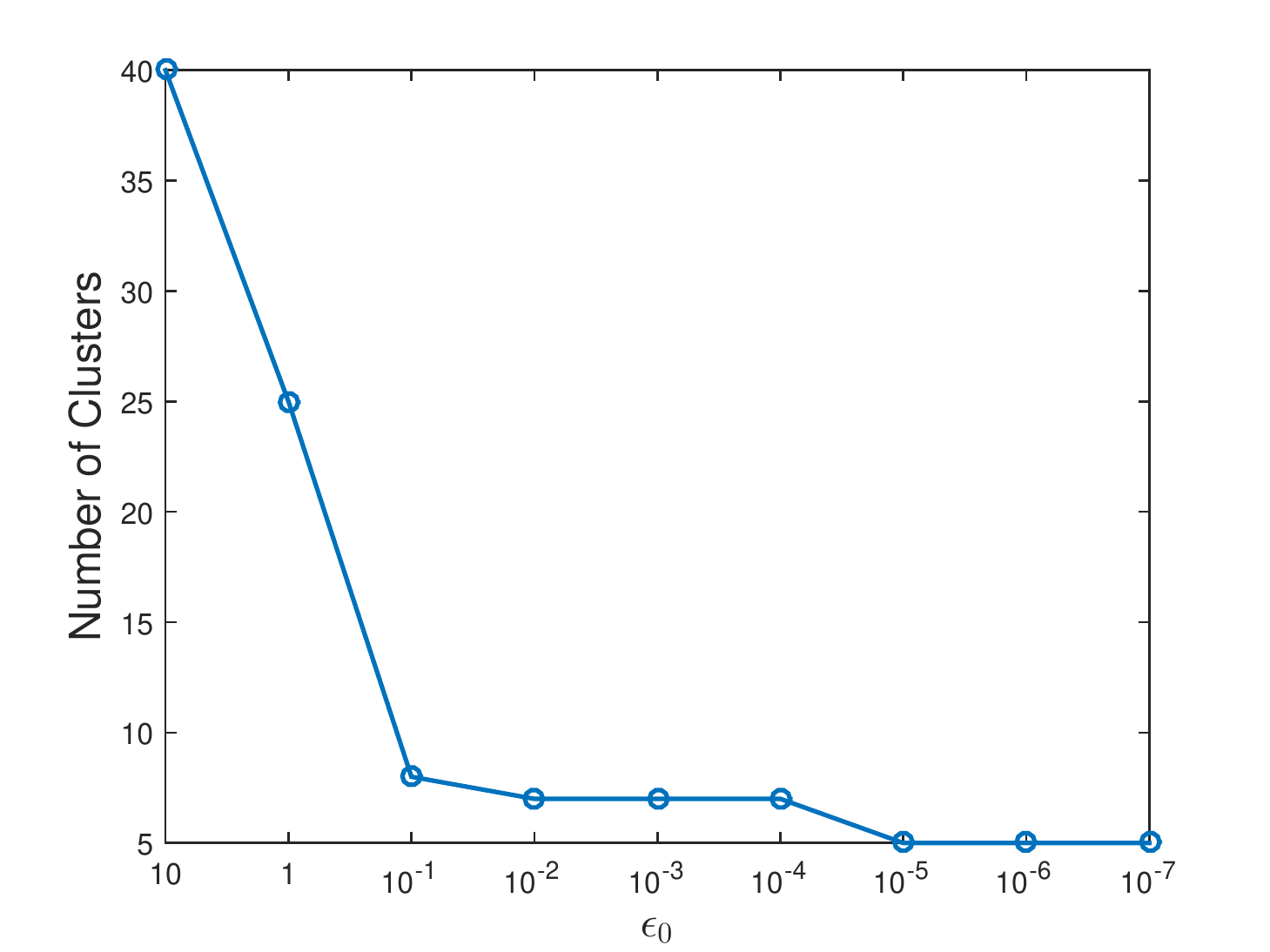}\label{Fig: par-2-2}}
     \subfloat[][Two crescents]{\includegraphics[width=.25\linewidth]{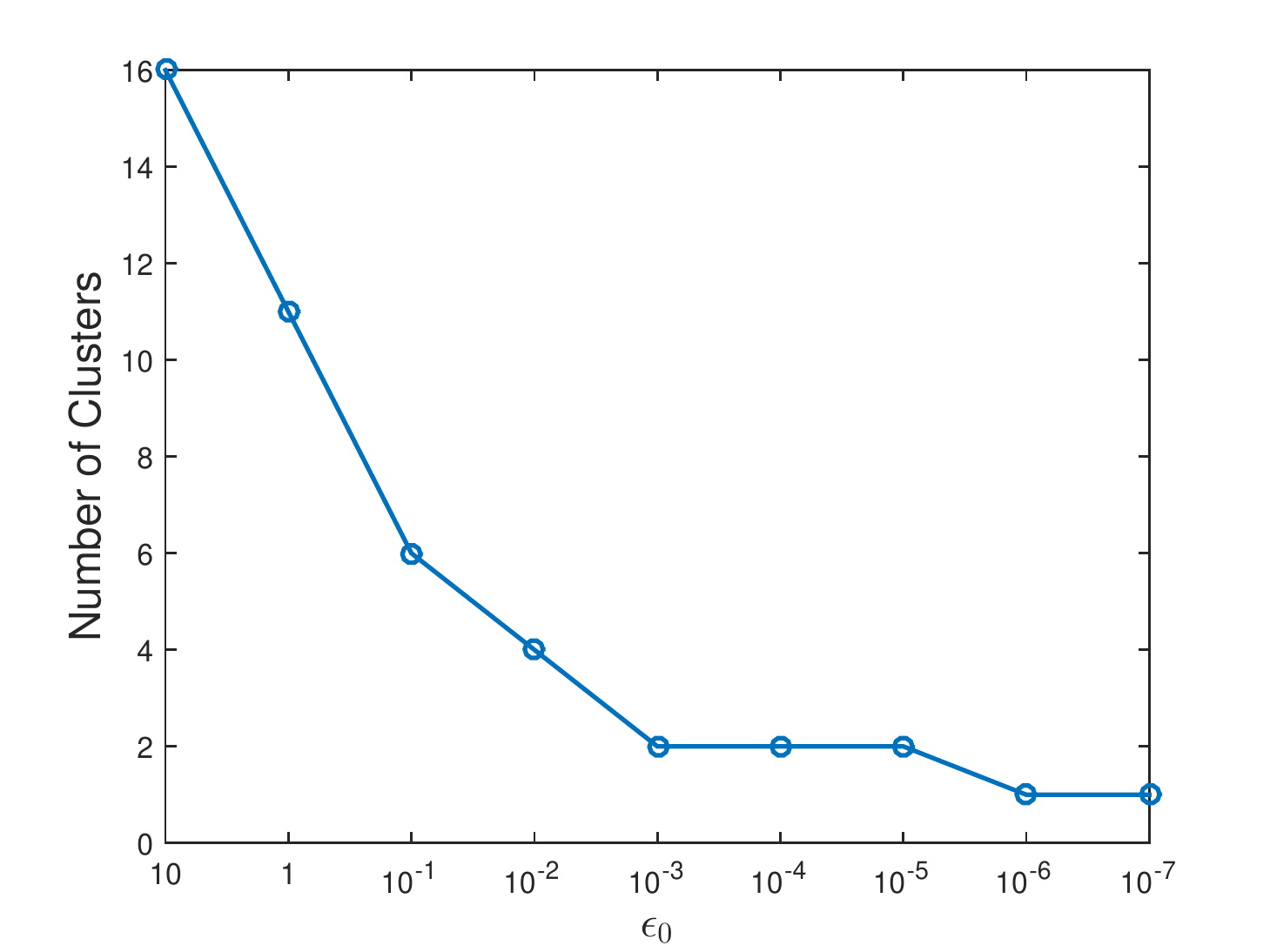}\label{Fig: par-2-3}}
     \subfloat[][Half ellipses]{\includegraphics[width=.25\linewidth]{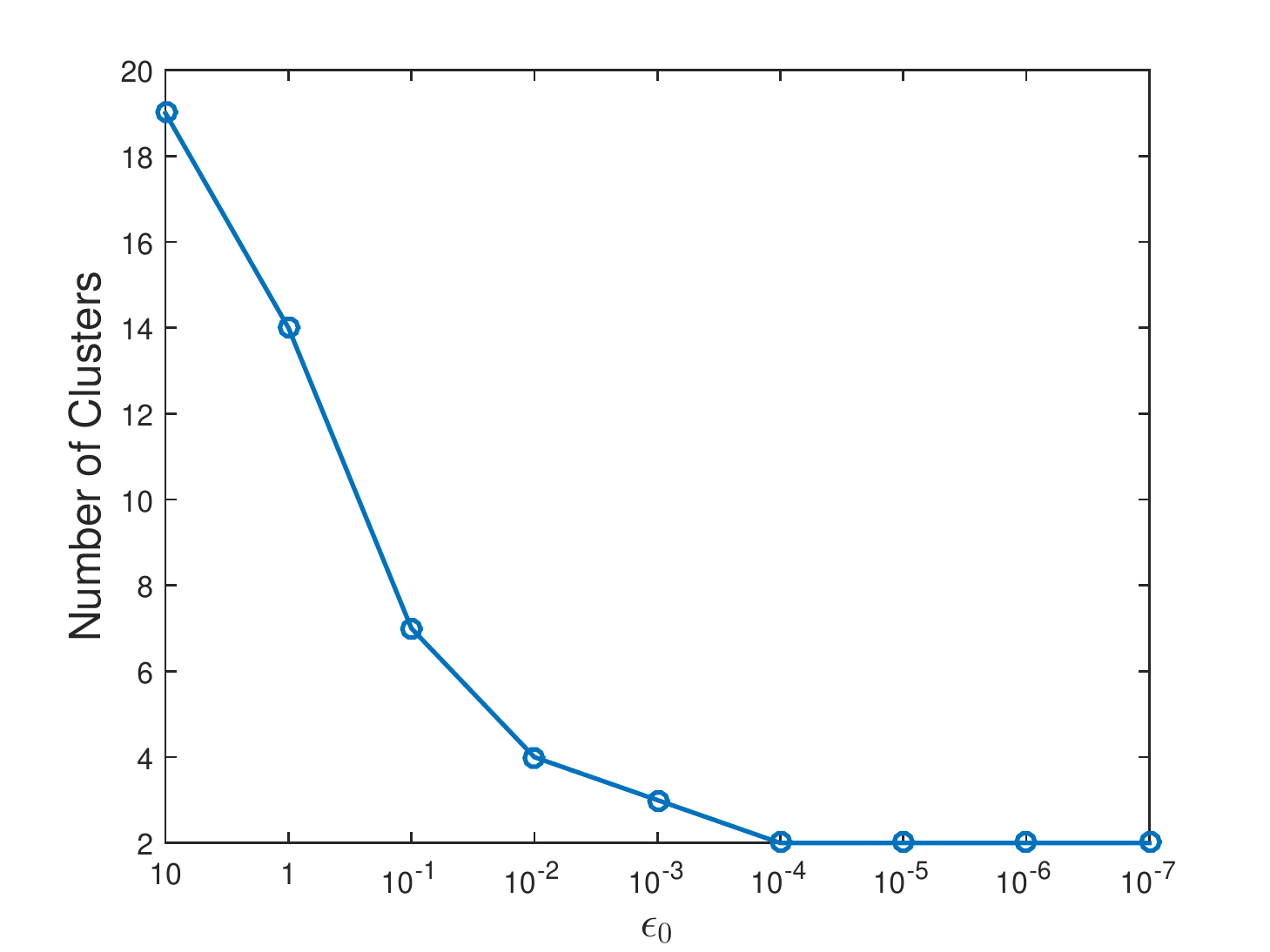}\label{Fig: par-2-4}}               
     
     \caption{$\epsilon_0$ vs Ncut and number of clusters in Algorithm \ref{Alg: agents-r}}
     \label{Fig: par-2}
\end{figure}

\subsection{Parameter selection}

Our algorithm proposed in this paper require only one parameter $\epsilon_0$. We start with some initial value of $\epsilon_0$ and monitor the number of clusters produced by RARD algorithm. We then decrease $\epsilon_0$ by a constant factor and observe the number clusters detected by RARD algorithm. In this fashion, we search for two consecutive $\epsilon_0$ with the same number of clusters, i.e., we set $\epsilon_0$, when the number of clusters become stable. Figure \ref{Fig: par-2} exhibits this procedure for synthetic datasets. If we already know the number of clusters $k$ in the dataset, then the value $\epsilon_0$ is chosen corresponding to the value of $k$. Gaussian similarity measure require a parameter $\sigma$ which is difficult to tune.  Automatically selecting $\sigma$ is a challenging problem and has received interest in some recent studies. More detailed analysis of this subject is beyond the scope this work. Interested readers can refer to \cite{Zelnik2004}.

\section{Connection to normalized cuts}
Shi and Malik proposed a graph partitioning criteria known as normalized cut (NCut) \cite{Shi2000}. A $k$-way NCut is defined as
\begin{equation}
\text{NCut}(\mathcal{V}_1, \mathcal{V}_2, \ldots, \mathcal{V}_k) = \sum_{i = 1}^k \frac{cut(\mathcal{V}_i, \mathcal{V} - \mathcal{V}_i)}{a(\mathcal{V}_i)},  \label{Eq: Ncut}
\end{equation}
where $cut(\mathcal{A}, \mathcal{B}) = \sum_{i \in \mathcal{A}, j \in \mathcal{B}} w_{ij}$ and $a(\mathcal{A}) = \sum_{i \in \mathcal{A}, j \in \mathcal{V}} w_{ij}.$ They showed that minimizing the 2-way normalized cut to obtain a bipartition of the graph is equal to minimizing the following Rayleigh quotient.
\[
\min_\mathbf{y} \frac{\mathbf{y}^T (D - W) \mathbf{y}}{\mathbf{y}^T D \mathbf{y}}
\]
where D is the degree matrix. The vector $\mathbf{y}$ of length $n$ satisfies $\mathbf{y}^TD\mathbf{1} = 0$ and $y_i \in \{1, -b\}$ with $b$ some constant in $(0,1).$ This problem is NP-Hard \cite{aggarwal2013} and is approximated by relaxing $\mathbf{y}$ to take on real values. This approximation leads to solving the generalized eigenvalue problem $(D - W)\mathbf{y} = \lambda D\mathbf{y}$ for the second smallest eigenvalue also known as the Fiedler value. Since $D$ is invertible the generalized eigenvalue problem is equivalent to solving
\begin{align}        
(I - D^{-1}W)\mathbf{y} = \lambda \mathbf{y}.	
\end{align}
Thus, minimizing a 2-way Ncut is approximated by finding the second smallest eigenvector of $I - D^{-1}W$. In this paper, we  are looking for a linear combination of the eigenvectors of $D^{-1}W$. Observe that eigenvectors of both these systems are same, however the eigenvalues are translated by 1.

\section{Conclusions}
We have proposed a fast spectral clustering algorithm based on a mixing process, which does not explicitly compute the eigenvectors of a similarity matrix. It rather finds an eigenvalue weighted linear combination of eigenvectors of the normalized similarity matrix. Our algorithms are simple to implement and computationally efficient. We have demonstrated the scalability and accuracy of the RARD algorithm by implementing it on large stochastic block models with tens of thousands of nodes and hundreds of millions of edges. We have also shown that our RARD algorithm has the same  accuracy as normalized cut algorithm. Thus our algorithm does not compromise on accuracy to achieve faster speed unlike other algorithms which claim to speed-up spectral clustering such as fast approximate spectral clustering \cite{Yan2009} and Nystrom method \cite{Fowlkes2004}.
\subsubsection*{Acknowledgments}
 This research is supported in part by NSF grant CNS 13-30077 and DMS 1312907.
\begingroup
\renewcommand{\section}[2]{}%
\subsubsection*{References}
\bibliographystyle{unsrt}
\bibliography{\myreferences}

\begin{thebibliography}{10}

\bibitem{Shi2000}
Jianbo Shi and Jitendra Malik.
\newblock Normalized cuts and image segmentation.
\newblock {\em IEEE Transactions on Pattern Analysis and Machine Intelligence},
  22(8):888--905, 2000.

\bibitem{Ding2001}
Chris~HQ Ding, Xiaofeng He, Hongyuan Zha, Ming Gu, and Horst~D Simon.
\newblock A min-max cut algorithm for graph partitioning and data clustering.
\newblock In {\em Proceedings of the IEEE International Conference on Data
  Mining}, pages 107--114, 2001.

\bibitem{Ng2002}
Andrew~Y Ng, Michael~I Jordan, and Yair Weiss.
\newblock On spectral clustering: Analysis and an algorithm.
\newblock In {\em Advances in Neural Information Processing Systems 14}, pages
  849--856. 2002.

\bibitem{Bach2004}
Francis~R. Bach and Michael~I. Jordan.
\newblock Learning spectral clustering.
\newblock In {\em Advances in Neural Information Processing Systems 16}, pages
  305--312. 2004.

\bibitem{Arbenz2012}
Peter Arbenz and Daniel Kressner.
\newblock Lecture notes on solving large scale eigenvalue problems.
\newblock {\em D-MATH, EHT Zurich}, 2012.

\bibitem{Fowlkes2004}
Charless Fowlkes, Serge Belongie, Fan Chung, and Jitendra Malik.
\newblock Spectral grouping using the {N}ystrom method.
\newblock {\em IEEE Transactions on Pattern Analysis and Machine Intelligence},
  26(2):214--225, 2004.

\bibitem{Sakai2009}
Tomoya Sakai and Atsushi Imiya.
\newblock Fast spectral clustering with random projection and sampling.
\newblock In {\em Machine Learning and Data Mining in Pattern Recognition},
  volume 5632 of {\em Lecture Notes in Computer Science}, pages 372--384. 2009.

\bibitem{Yan2009}
Donghui Yan, Ling Huang, and Michael~I Jordan.
\newblock Fast approximate spectral clustering.
\newblock In {\em Proceedings of the 15th ACM SIGKDD International Conference
  on Knowledge Discovery and Data Mining}, pages 907--916, 2009.

\bibitem{Wen-yen2011}
Wen-Yen Chen, Yangqiu Song, Hongjie Bai, Chih-Jen Lin, and Edward~Y Chang.
\newblock Parallel spectral clustering in distributed systems.
\newblock {\em IEEE Transactions on Pattern Analysis and Machine Intelligence},
  33(3):568--586, 2011.

\bibitem{Chung1997}
Fan~RK Chung.
\newblock {\em Spectral graph theory}, volume~92.
\newblock American Mathematical Soc., 1997.

\bibitem{Simon1961}
Herbert~A Simon and Albert Ando.
\newblock Aggregation of variables in dynamic systems.
\newblock {\em Econometrica: Journal of the Econometric Society}, pages
  111--138, 1961.

\bibitem{Le1990}
Yann Le~Cun, Bernhard~E Boser, John~S Denker, Donnie Henderson, Richard~E
  Howard, Wayne~E. Hubbard, and Lawrence~D Jackel.
\newblock Handwritten digit recognition with a back-propagation network.
\newblock In {\em Advances in Neural Information Processing Systems 2}, pages
  396--404. 1990.

\bibitem{Nene1996}
Sameer~A Nene, Shree~K Nayar, and Hiroshi Murase.
\newblock Columbia object image library (coil-20).
\newblock Technical report, CUCS-005-96, 1996.

\bibitem{Zelnik2004}
Lihi Zelnik-Manor and Pietro Perona.
\newblock Self-tuning spectral clustering.
\newblock In {\em Advances in Neural Information Processing Systems 17}, pages
  1601--1608. 2005.

\bibitem{aggarwal2013}
Charu~C Aggarwal and Chandan~K Reddy.
\newblock {\em Data clustering: algorithms and applications}.
\newblock CRC Press, 2013.

\bibitem{Gionis2005}
Aristides Gionis, Heikki Mannila, and Panayiotis Tsaparas.
\newblock Clustering aggregation.
\newblock {\em ACM Transactions on Knowledge Discovery from Data}, 1(1):4,
  2007.

\end{thebibliography}
\endgroup

\newpage

\appendix

\section{Details of Datasets}
\label{Ap: Exam}
\subsection{Synthetic Datasets}
\begin{itemize}
\item \emph{Mixture of Gaussians:} We consider a mixture of five Gaussians random variables $X_1, X_2,X_3,X_4$ and $X_5$ with different densities. The five mean vectors and covariance matrices are
\begin{align*}
&\mathbf{\mu}_1 =  \left[ \begin{array}{r}
-5 \\
-5
\end{array}\right],~~\mathbf{\mu}_2 =  \left[ \begin{array}{r}
0 \\
0
\end{array}\right],~~
\mathbf{\mu}_3 =  \left[ \begin{array}{r}
6 \\
-6
\end{array}\right],~~\mathbf{\mu}_4 =  \left[ \begin{array}{r}
-6 \\
6
\end{array}\right],~~\mathbf{\mu}_5 =  \left[ \begin{array}{r}
5 \\
5
\end{array}\right] \\
&\Sigma_1 =  \left[ \begin{array}{r@{\quad} r}
.5 & 0 \\
0 & .5
\end{array}\right],~~\Sigma_2 =  \left[ \begin{array}{r@{\quad} r}
3.5 & 0 \\
0 & 3.5
\end{array}\right],~~
\Sigma_3 =  \left[  \begin{array}{r@{\quad} r}
2 & 0 \\
0 & 2
\end{array}\right],~~ \\
&\Sigma_4 =  \left[  \begin{array}{r@{\quad} r}
1 & 0 \\
0 & 1
\end{array}\right],~~\Sigma_5 =  \left[  \begin{array}{r@{\quad} r}
1 & -.5\\
-.5 & 1.5
\end{array}\right] 
\end{align*}
A sample of 2000 points is taken from these distributions with $100$,  1000, 300, 200 and 400 samples from $X_1, X_2,X_3,X_4$ and $X_5$ respectively. 
We have used $\sigma = 0.5$ for RBF similarity.
\item \emph{Clustering aggregation:} This dataset set is taken from Clustering aggregation paper \cite{Gionis2005}. The authors show that single link, complete link, average link, ward's method and k-means fail to recover the correct clusters in this data set. It has 7 clusters and 788 total points. Parameter $\sigma$ is kept at 1 for RBF similarity function.
\item \emph{Two crescents:} This dataset has two clusters with a total of 384 points. The value of parameter $\sigma$ in the RBF similarity function is kept equal to 1.5. 
\item \emph{Half ellipses:} This dataset contains a total of 2000 points with 1000 points in each cluster. Parameter $\sigma = 2.5$ is used for RBF similarity.
\end{itemize}
\subsection{Real Datasets}
\begin{itemize}
\item \emph{USPS} \cite{Le1990}: This dataset contains $7291$ grayscale images of digits $0-9$ scanned from envelopes by United States Postal Service. Each feature vector consists of normalized grayscale values of $16\times16=256$ pixels.
\item \emph{COIL20} \cite{Nene1996}: This dataset consists of $32 \times 32$ grayscale images of 20 different objects. 72 different images of each object are taken from 72 different angles as the objects are rotated on a table, i.e., after every 5 degree rotation an image is taken. A feature vector consists of 1024 normalized grayscale values i.e., one value for each pixel. 
\end{itemize}
\begin{figure}[t!]
\centering
\includegraphics[scale=0.75]{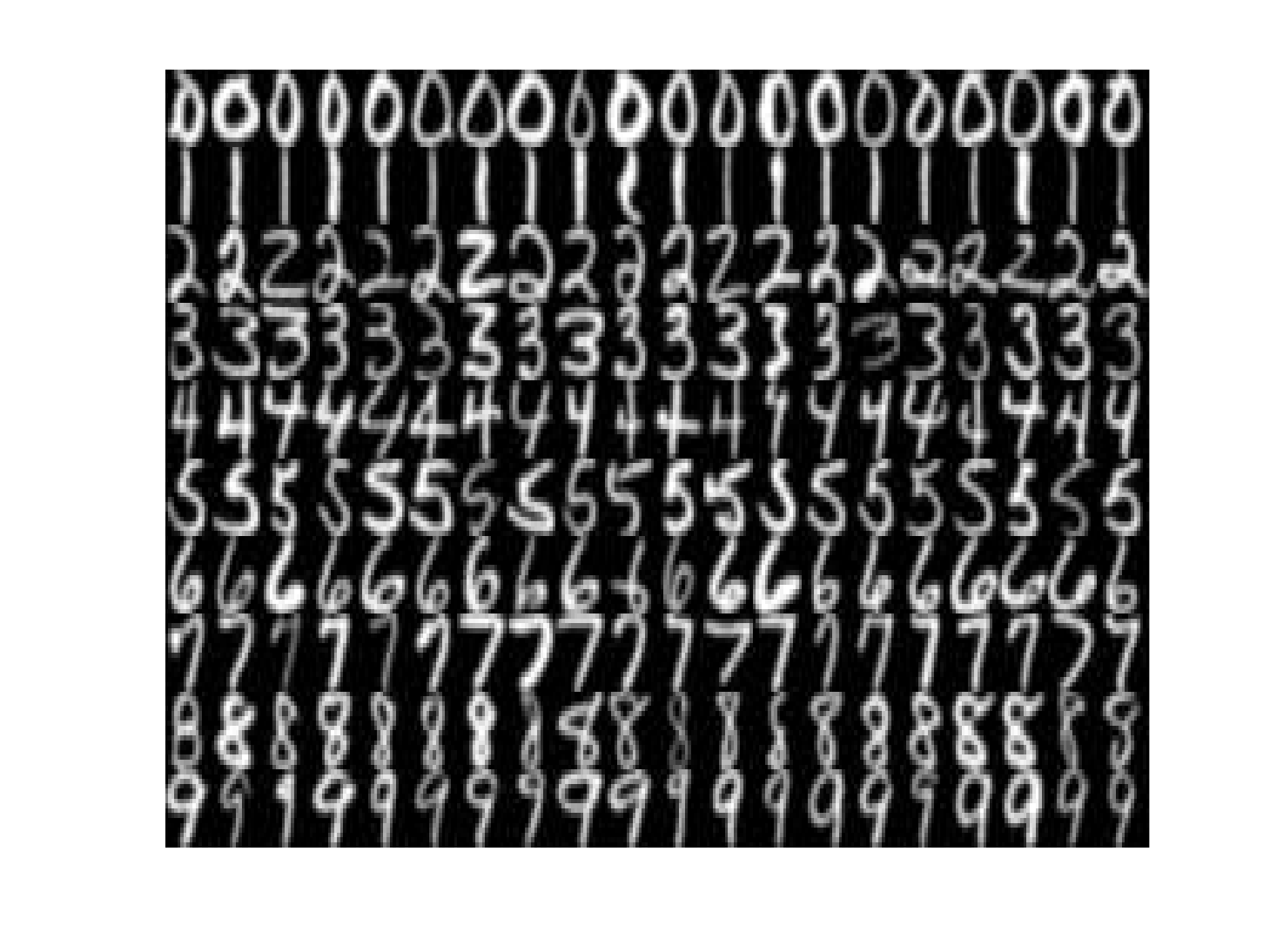}
\caption{A sample of USPS dataset}
\label{Fig: usps_samp}
\end{figure}
\begin{figure}[t!]
\centering
\includegraphics[scale=0.75]{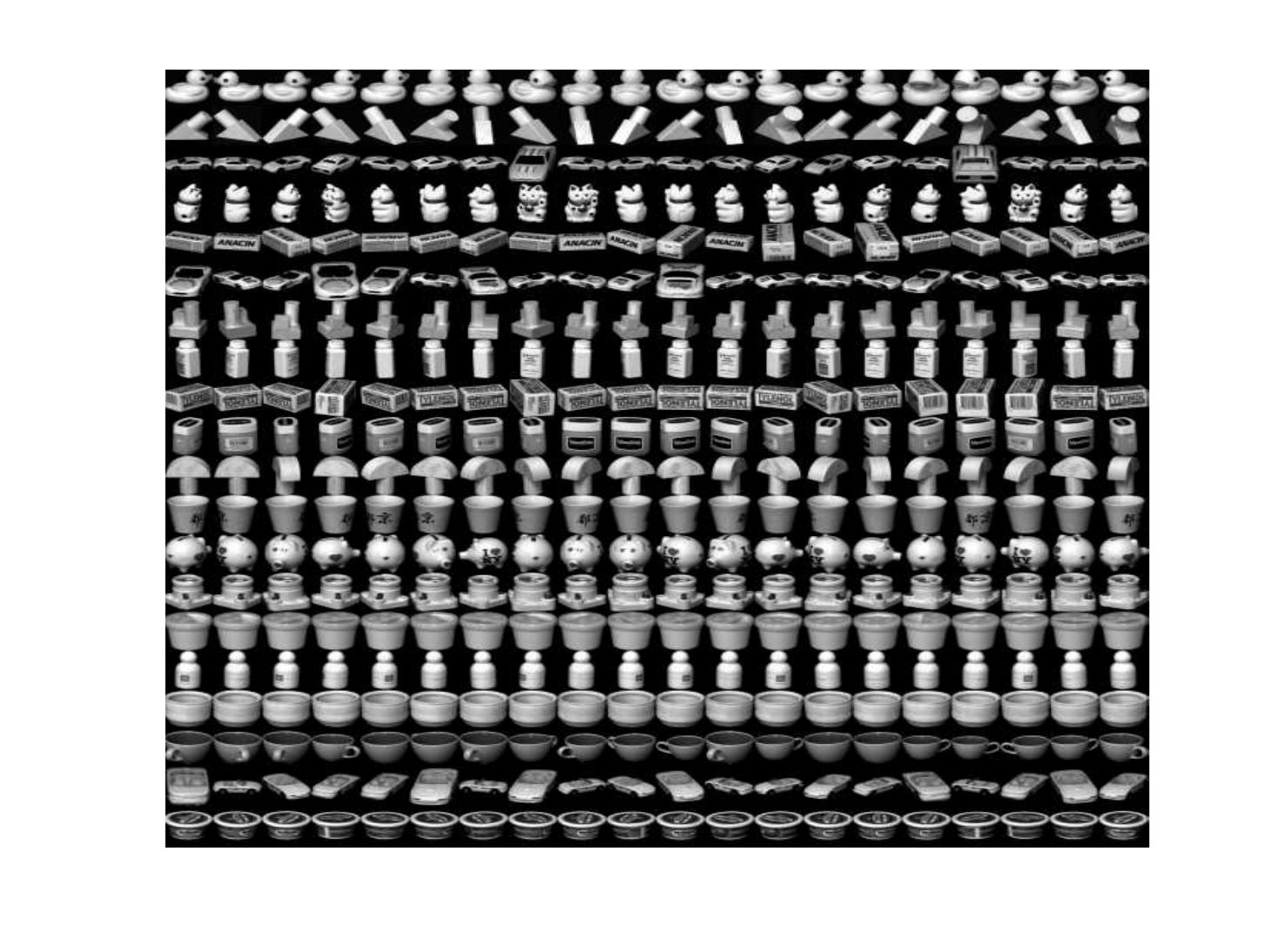}
\caption{A sample of COIL20 dataset}
\label{Fig: coil_samp}
\end{figure}


\end{document}